\documentclass[10pt, conference, letterpaper]{IEEEtran}
\usepackage{amsmath}
\usepackage{color}
\usepackage{multicol}
\usepackage{array}
\usepackage{stfloats}
\usepackage{amsthm} 
\usepackage{multirow}
\usepackage{amssymb}
\usepackage{float}
\usepackage{nccmath}
\usepackage{graphicx}
\usepackage{setspace}
\usepackage{booktabs}
\usepackage{mwe}
\usepackage{cite}
\usepackage{amsmath,amssymb,amsfonts}
\usepackage{graphicx}
\usepackage{textcomp}
\usepackage{xcolor}
\usepackage{physics}
\usepackage{bm}
\usepackage[utf8]{inputenc} 
\usepackage{kotex} 
\usepackage[linesnumbered,ruled]{algorithm2e}

\usepackage{algorithmicx}
\usepackage{xspace}
\usepackage{subfigure}
\usepackage{hyperref}
\usepackage{kotex}
\usepackage{tikz}

\newcommand{\BfPara}[1]{{\noindent\bf#1.}\xspace}

\makeatletter
\newsavebox{\@brx}
\newcommand{\llangle}[1][]{\savebox{\@brx}{\(\m@th{#1\langle}\)}%
  \mathopen{\copy\@brx\kern-0.5\wd\@brx\usebox{\@brx}}}
\newcommand{\rrangle}[1][]{\savebox{\@brx}{\(\m@th{#1\rangle}\)}%
  \mathclose{\copy\@brx\kern-0.5\wd\@brx\usebox{\@brx}}}
\makeatother

\newtheorem{assumption}{Assumption}
\newtheorem{lemma}{Lemma}
\newtheorem{theorem}{Theorem}

\newtheorem{proposition}{Proposition}

\usepackage{amsmath,amsfonts,bm}

\def\eqref#1{eqn.~(\ref{#1})}

\def\1{\bm{1}}

\DeclareMathAlphabet{\mathsfit}{\encodingdefault}{\sfdefault}{m}{sl}
\SetMathAlphabet{\mathsfit}{bold}{\encodingdefault}{\sfdefault}{bx}{n}

\def\tt{{\bf t}}

\def\0{{\bf 0}}
\def\1{{\bf 1}}

\def\EB{{\mathbb E}}

\usepackage{fancyhdr}
\begin{document}

\title{\fontsize{23}{28}\selectfont Joint Superposition Coding and Training for \\Federated Learning over Multi-Width Neural Networks}

\author{
\IEEEauthorblockN{$^{\dag}$Hankyul Baek, $^{\dag}$Won Joon Yun, $^{\dag}$Yunseok Kwak, $^{\ddag}$Soyi Jung, $^{\circ}$Mingyue Ji, $^{\ast}$Mehdi Bennis, \\ $^{\diamond}$Jihong Park, and $^{\dag}$Joongheon Kim}
\IEEEauthorblockA{$^{\dag}$Department of Electrical and Computer Engineering, Korea University, Seoul, Republic of Korea
\\
$^{\ddag}$School of Software, Hallym University, Chuncheon, Republic of Korea
\\
$^{\circ}$Department of Electrical and Computer Engineering, University of Utah, Salt Lake City, UT, USA
\\
$^{\ast}$Centre for Wireless Communications, University of Oulu, Oulu, Finland
\\
$^{\diamond}$School of Information Technology, Deakin University, Geelong, Australia
}
E-mails: \texttt{\{67back,ywjoon95,rhkrdbstjr0\}@korea.ac.kr}, \texttt{sjung@hallym.ac.kr}, 
\texttt{mingyue.ji@utah.edu}, \\ 
\texttt{mehdi.bennis@oulu.fi}, \texttt{jihong.park@deakin.edu.au}, \texttt{joongheon@korea.ac.kr}
}
\maketitle

\begin{abstract}
This paper aims to integrate two synergetic technologies, federated learning (FL) and width-adjustable slimmable neural network (SNN) architectures. FL preserves data privacy by exchanging the locally trained models of mobile devices. By adopting SNNs as local models, FL can flexibly cope with the time-varying energy capacities of mobile devices. Combining FL and SNNs is however non-trivial, particularly under wireless connections with time-varying channel conditions. Furthermore, existing multi-width SNN training algorithms are sensitive to the data distributions across devices, so are ill-suited to FL. Motivated by this, we propose a communication and energy efficient SNN-based FL (named \emph{SlimFL}) that jointly utilizes \emph{superposition coding (SC)} for global model aggregation and \emph{superposition training (ST)} for updating local models. By applying SC, SlimFL exchanges the superposition of multiple width configurations that are decoded as many as possible for a given communication throughput. Leveraging ST, SlimFL aligns the forward propagation of different width configurations, while avoiding the inter-width interference during back propagation. We formally prove the convergence of SlimFL. The result reveals that SlimFL is not only communication-efficient but also can counteract non-IID data distributions and poor channel conditions, which is also corroborated by simulations.\end{abstract}

\section{Introduction}\label{sec:1}

  Federated learning (FL) is a promising solution to enable high-quality on-device learning at mobile devices such as phones, cars, and drones~\cite{Brendan17,infocom1}. Each of these devices has only a limited amount of local data, and FL can overcome the lack of local model training samples by exchanging and aggregating the local models of different devices. To reach its full potential, it is essential to scale up the range of federating devices that are often wirelessly connected while having heterogeneous levels of available energy~\cite{pieee21park}. This mandates addressing the following interrelated energy and wireless communication problems.

\begin{figure}[t!]
\centering
    \subfigure[SlimFL.]{\includegraphics[width=.46\columnwidth]{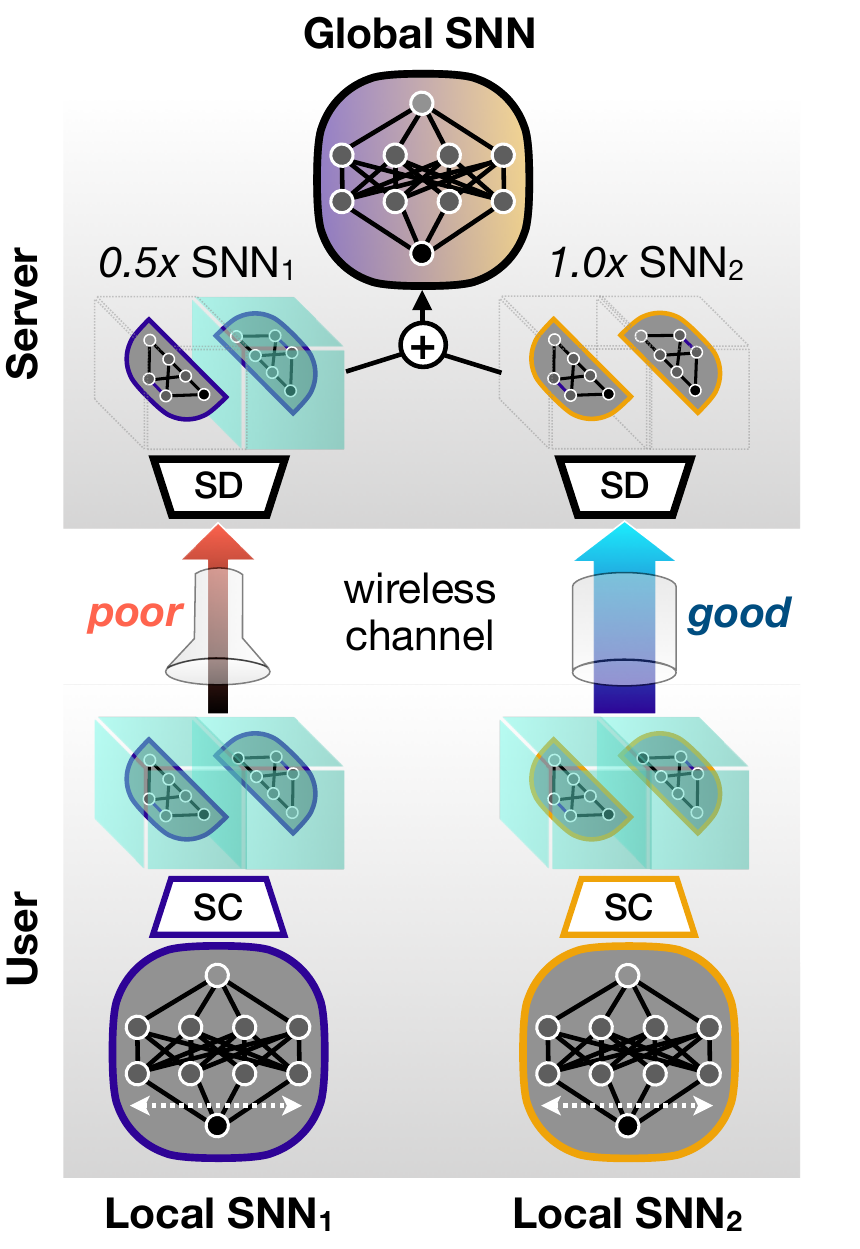}}
    \subfigure[Vanilla FL-1.5x.]{\includegraphics[width=.52\columnwidth]{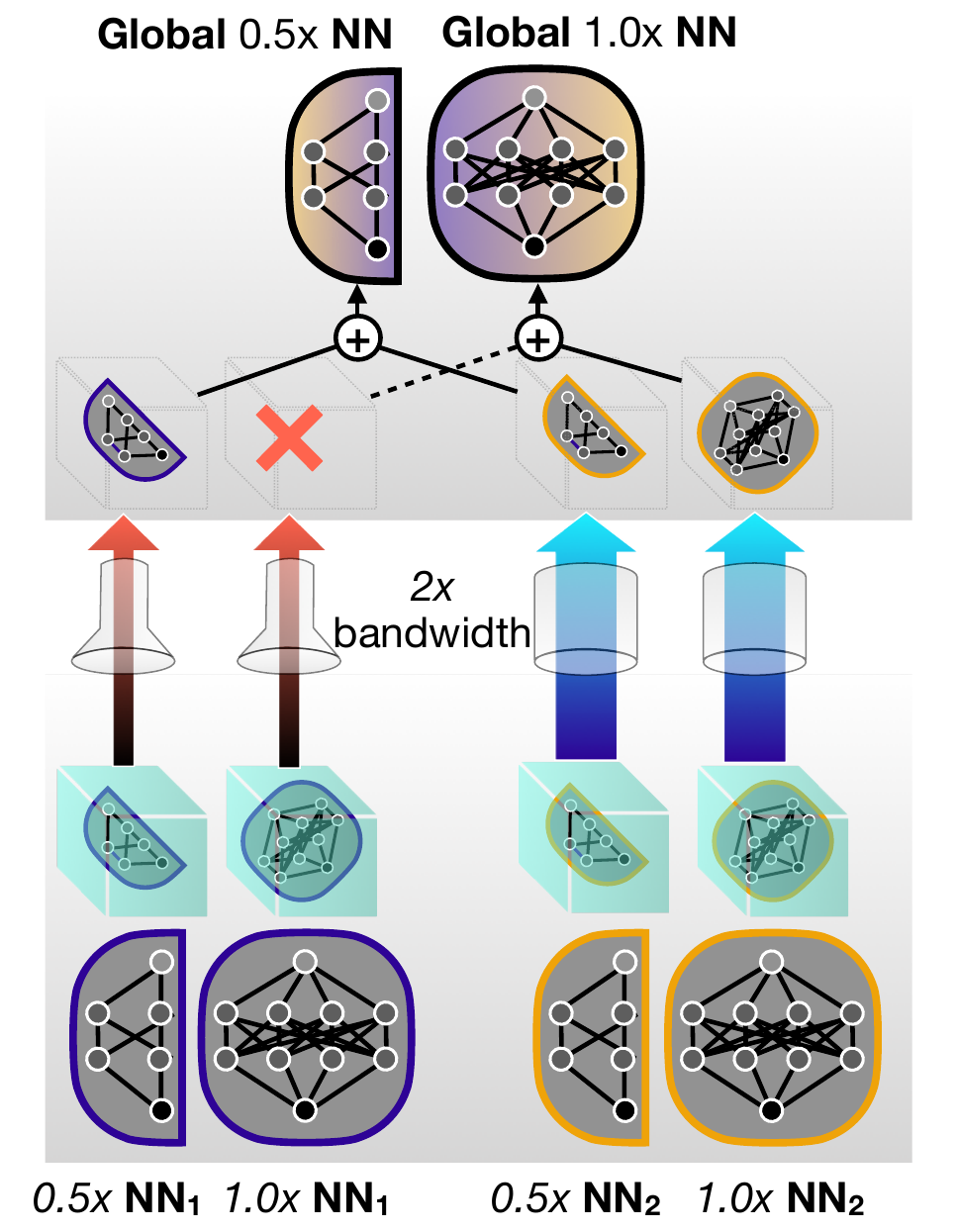}}
    \caption{A schematic illustration of (a) \emph{slimmable federated learning (SlimFL)} using the slimmable neural networks (SNNs) with the superposition coding (SC) and successive decoding (SD), compared to (b) vanilla federated learning (Vanilla FL-1.5x) consuming 2x bandwidth.}
    \label{fig:abstract}
    \vspace{-3mm}
\end{figure}

On the one hand, different devices have heterogeneous levels of available energy. Low-energy devices are likely to run small models, whereas high-energy devices prefer to operate large models. Unfortunately, FL can only aggregate the local models under the same architecture \cite{Brendan17}, so is able to train either small or large models at a time. To cope with heterogeneous energy capacity, one should therefore perform FL two times with the increased overall training time, or run FL simultaneously for two separate groups of devices with reduced training samples while compromising accuracy. Performing FL using a width-controllable \emph{slimmable neural network (SNN)} architecture enables to train these two-level models at once while federating across all devices, after which each trained local SNN model can adjust its width\typeout{ due to its desired energy consumption} \cite{ICCV2019_USlimmable}.

On the other hand, wireless communication channel conditions vary over time and across different devices. When the channel information is known before transmission, poor-channel devices can only exchange small models, while good-channel devices can participate in FL using large models. SNN allows them to collaborate together, in a way that poor-channel devices send their local SNNs after reducing the widths, and contribute only to a fraction of the entire global model construction. This however entails extra communication and energy costs for probing channel conditions \cite{4267831} that change over time and locations due to random fading and mobility.

Spurred by the aforementioned problems, we propose the first \emph{SNN-based FL algorithm} that leverages \emph{superposition coding (SC)} and \emph{successive decoding (SD)}, coined \emph{slimmable FL (SlimFL)}. By applying SNNs to FL, SlimFL can address the heterogeneous energy capacities. Besides, by exploiting SC and SD, SlimFL can proactively cope with the heterogeneous channel conditions for unknown channel state information. 

To illustrate, consider an SNN with two width levels as shown in Fig.~\ref{fig:abstract}. In the uplink from each device to the server, the device uploads its local updates after jointly encoding the left-half (LH) and the right-half (RH) of its local SNN model while allocating different transmission power levels to them, i.e., SC \cite{Cover:TIT72}. Then, the server first attempts to decode the LH. If decoding the LH is successful, the server successively tries to decode the RH, i.e., SD or also known as successive interference cancellation (SIC). Accordingly, when the device-server channel {throughput} is low, the server can decode only the LH of the uploaded model, obtaining the \emph{half-width (0.5x)} model. When the channel {throughput} is high, the server can decode both LH and RH, and combine them to yield the \emph{full-width model (1.0x)}. Consequently, the sever constructs a global model superpositioning the decoded 0.5x and 1.0x local models, which is downloaded by each device. The device replaces its local model with the downloaded global model, and iterates the aforementioned operation until convergence.

In essence, the effectiveness of SlimFL hinges on creating a synergy between multiple width configurations, i.e., 0.5x and 1.0x models, which is however non-trivial. The global model is a mixture of different width configurations, so the standard FL convergence becomes questionable. Furthermore, the local model consists of multiple width configurations, so training them may interfere with one another. Existing SNN architectures and training algorithms are intended for standalone learning, so are ill-suited for SlimFL particularly under non-independent and identically distributed (non-IID) data distributions. To address these challenges in SlimFL, in this paper we develop novel SNN architecture and and training algorithm, named \emph{superposition training (ST)}, and study the convergence and effectiveness of SlimFL. The major contributions of this paper are summarized as below.
\begin{enumerate}

\item We first propose an FL framework for SNNs, SlimFL (see Fig. 1(a) and \textbf{Algorithm~\ref{alg:SlimFL}}), which exploits SC for improving communication efficiency under time-varying wireless channels with limited bandwidth.

\item We develop a local SNN training method for SlimFL, ST (see \textbf{Algorithm~\ref{alg:sustrain}}), which avoids unnecessary inter-width interference, and thus achieves fast convergence with high accuracy regardless of data distributions.

\item We propose an energy-efficient SNN architecture, Ultra Light MobileNet (see \textbf{Table \ref{tab:ULSNN}}), achieving $>\!36$x less FLOPS than the state-of-the-art SNN architecture~\cite{ICCV2019_USlimmable}.

\item We prove the convergence of SlimFL (see \textbf{Theorem~\ref{convth2}}). The result shows the favorable conditions of SlimFL in terms of the channel quality and data distributions, and provides the optimal transmit power allocation guideline on SC (see \textbf{Proposition~1}) as well as the optimal weight guideline on ST (see \textbf{Proposition~2}).

\item We corroborate our analysis by simulation, showing that compared to vanilla FL (see Fig. 1(b)), SlimFL achieves higher accuracy and lower communication costs under poor channel conditions and non-IID data distributions.
\end{enumerate}

The notations in this paper are listed in Tab.~\ref{tab:notation-convergence}.

\section{Related Work}\label{sec:2}

\subsection{Multi-Width/Depth Neural Networks}  
To meet different on-device energy and memory requirements, it is common to prune model weights \cite{Han:16} or transfer a large trained model's knowledge into a small empty model via knowledge distillation (KD) \cite{HintonKD:14}, which however incurs additional training operations. Alternatively, one can adjust a trained model's width and/or depth in accordance with the resource requirements. Following this principle, depth-controlled neural networks \cite{IJCNN2019_DepthControllable} and adaptive neural networks \cite{AAAI2019_Anytime} can adjust their depths after training, whereas SNNs tune their widths \cite{ICCV2019_USlimmable}. In this paper, we leverage width-controllable SNNs, and develop its FL version, SlimFL. Such an extension is non-trivial, and entails several design issues, such as local SNN training algorithms, aggregating segment prioritization\typeout{ (e.g., more aggregating the same LH/RH segments vs. balancing the LH and RH aggregations), which will be discussed in Sec~\ref{sec:4}}.

\subsection{Superposition Coding \& Successive Decoding}
In a nutshell, SC encodes two different data signals into one while allocating two different power levels before transmissions \cite{Cover:TIT72}. After receptions, SD decodes the SC-encoded signal by first decoding the stronger signal, followed by subtracting it and decoding the remainder as the weaker signal \cite{TseBook:FundamaentalsWC:2005}. SC has been widely utilized in communication systems, particularly for simultaneously supporting different devices in the context of non-orthogonal multiple access (NOMA) \cite{7842433}. We apply the same principle for supporting a single device simultaneously requesting two types of data with different priorities, such that the higher priority signal should almost surely be decoded while the lower priority signal can be successively decoded only under good channel conditions. Precisely, SlimFL makes an SNN's LH a higher priority so as to receive the $0.5$x model even under poor channels. It can decode the SNN's RH only when the channel conditions are good, obtaining the $1.0$x model by combining both LH and RH. Consequently, SlimFL ensures stable convergence under poor channels. 

\subsection{FL Convergence Analysis}
FL convergence has recently been studied extensively~\cite{wang2021field,khaled2020tighter}, among which we fundamentally rely on the following convergence results. Under IID data distributions, vanilla FL, also known as FedAvg, is equivalent to the local SGD algorithm whose convergence is known~\cite{mangasarian1995parallel}. Under non-IID data distributions, the convergence of FedAvg is provided in \cite{li2019convergence} where the non-IIDness is determined by the bound of a dissimilarity between global and local average risks. Alternatively, the convergence under non-IID data distributions is proved by \cite{khaled2020tighter} where the non-IIDness is measured by the average of the local stochastic gradient variance, where the average is taken across the devices. Taking into account non-IID data distributions, our SlimFL convergence analysis relies primarily on the method in \cite{li2019convergence}, but our non-IIDness definition is similar to \cite{khaled2020tighter}. Note that the convergence of standalone SNN training has recently been studied in \cite{Stich2106}, yet without FL.

\section{Local Model Architecture and Training}\label{sec:3}
Existing SNN architectures and training algorithms are intended for standalone learning~\cite{ICCV2019_USlimmable}. This section proposes a novel SNN architecture and its local training for SlimFL.

\subsection{Ultra Light SNN Architecture}\label{sec:3-1}

\begin{table}[t!]
    \caption{Model architecture of UL--MobileNet.}
    
    \small \centering
    \resizebox{\columnwidth}{!}{\begin{tabular}{l|r}
    \toprule[1pt]
      \bf{UL--MobileNet Layers}             & \bf{Weight connection of layers 1.0x (0.5x)}  \\ \midrule
      \multicolumn{1}{l|}{\bf{Convolution layer}} &$C_{in}\times C_{out} \times K \times K$\\
      Conv2D + ReLU6 & $1 \times 32 \times 3 \times 3~(1 \times 16 \times 3 \times 3)$  \\
      Conv2D + ReLU6 & $1 \times 32 \times 3 \times 3~(1 \times 16 \times 3 \times 3)$  \\
      Conv2D + ReLU6 & $32 \times 32 \times 1 \times 1~(16 \times 16 \times 1 \times 1)$\\
      Conv2D + ReLU6 & $1 \times 32 \times 3 \times 3~(1 \times 16 \times 3 \times 3)$  \\
      Conv2D + ReLU6 & $32\times 64 \times 1 \times 1~(16 \times 32 \times 1 \times 1)$ \\
      \multicolumn{1}{l|}{\bf{Fully connected layer}}& $C_{in}\times C_{out}$ \\
      Linear         & $64 \times 10~(32 \times 10)$ \\ \bottomrule[1pt] 
    \end{tabular}}
    \begin{tabular}{p{0.98\columnwidth}}
    \footnotesize $C_{in}$, $C_{out}$, and $K$ stand for dimension of input channel, dimension of output channel, kernel size, respectively, and $\text{ReLU}6(x) = \min(\max(0, x), 6)$.
    \end{tabular}\label{tab:ULSNN}
    \end{table}
The state-of-the-art SNN architecture is the {US-MobileNet} proposed in~\cite{ICCV2019_USlimmable}. As opposed to a \emph{de facto} standard neural network architecture with a universal batch normalization (BN) layer, US-MobileNet is equipped with multiple separate BN layers to cope with all slimmable model configurations. While effective in standalone learning, in SlimFL with wireless connectivity, not all multi-width configurations are exchanged due to insufficient communication throughput, while the exchanged width configurations are aggregated across devices, diluting the effectiveness of BN. In our experiments we even observed training convergence failures due to BN. Furthermore, managing multiple BN layers not only consumes additional memory costs, but also entails high computing computing overhead. For these reasons, we remove BN layers, and consider a lighter version of US-MobileNet, named \emph{Ultra Light MobileNet (UL-MobileNet)}, with the specifics provided in Tab.~\ref{tab:tab_parameters}. Compared to US-MobileNet with more than 100M FLOPS, UL-MobileNet costs only 2.76M FLOPS. 

Hereafter, we consider that each device $k$ has an SNN following the UL-MobileNet architecture. At the $t$-th iteration, the SNN model has the weight vector $\theta_t^k$ with two width configurations: 0.5x width configuration $\theta_t^k \odot \Xi_1$ and 1.0x width configuration $\theta_t^k \odot \Xi_2$ ($= \theta_t^k$), where $\odot$ is the element-wise product and $\Xi_i$ represents a binary mask for extracting the parameters of $i$-th width configuration.

\subsection{Superposition SNN Training}\label{superposition}

Training a multi-width SNN is challenging, in that the weights of earlier trained width configurations can be distorted by the latter backpropagation (BP) for other overlapping widths. This inter-width interference not only deteriorates the inference accuracy, while hindering the training convergence. The first SNN training algorithm proposed in \cite{yu2018slimmable}, referred to as SlimTrain (see Appendix~\ref{alg:slim}), partly mitigates such inter-width interference by training different width configurations in descending order of size. 
While following the sample principle, the state-of-the-art SNN training algorithm proposed in \cite{ICCV2019_USlimmable}, referred to as universal SNN (USTrain), additionally applies the inplace knowledge distillation (IPKD) from the full-width to all the sub-widths. The IPKD encourages each sub-width (i.e., student) to yield a similar softmax output (i.e., logit) to that of the full-width (i.e., teacher) such that their overlapping BP gradients become less different from each other, thereby reducing the inter-width interference.

However, as shown by experiments in Fig. \ref{fig:trainmethod}, USTrain is unfit for SlimFL particularly under non-IID data distributions (i.e., $\alpha=0.1$, see Fig.~\ref{fig:Non-iid-distribution}), where SlimTrain even outperforms USTrain. We conjecture that the problem comes from the use of outdated teacher's knowledge in USTrain. In USTrain, the teacher's logit is set as the value before updating the teacher's model, and is compared with a student after updating the teacher's model. Non-IID data distributions exacerbate this mismatch, where the full-width teacher model is significantly updated in the first epoch after downloading the global model due to the huge gap between local and global models.

\begin{figure}[t!]
\centering
\begin{tabular}{cc}
     \multicolumn{2}{c}{\includegraphics[width=0.85\columnwidth]{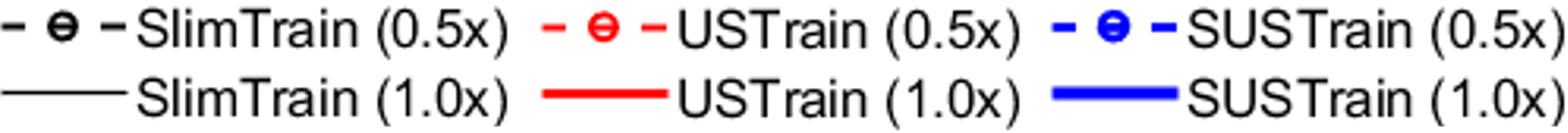}}\\
    \includegraphics[width=0.45\columnwidth]{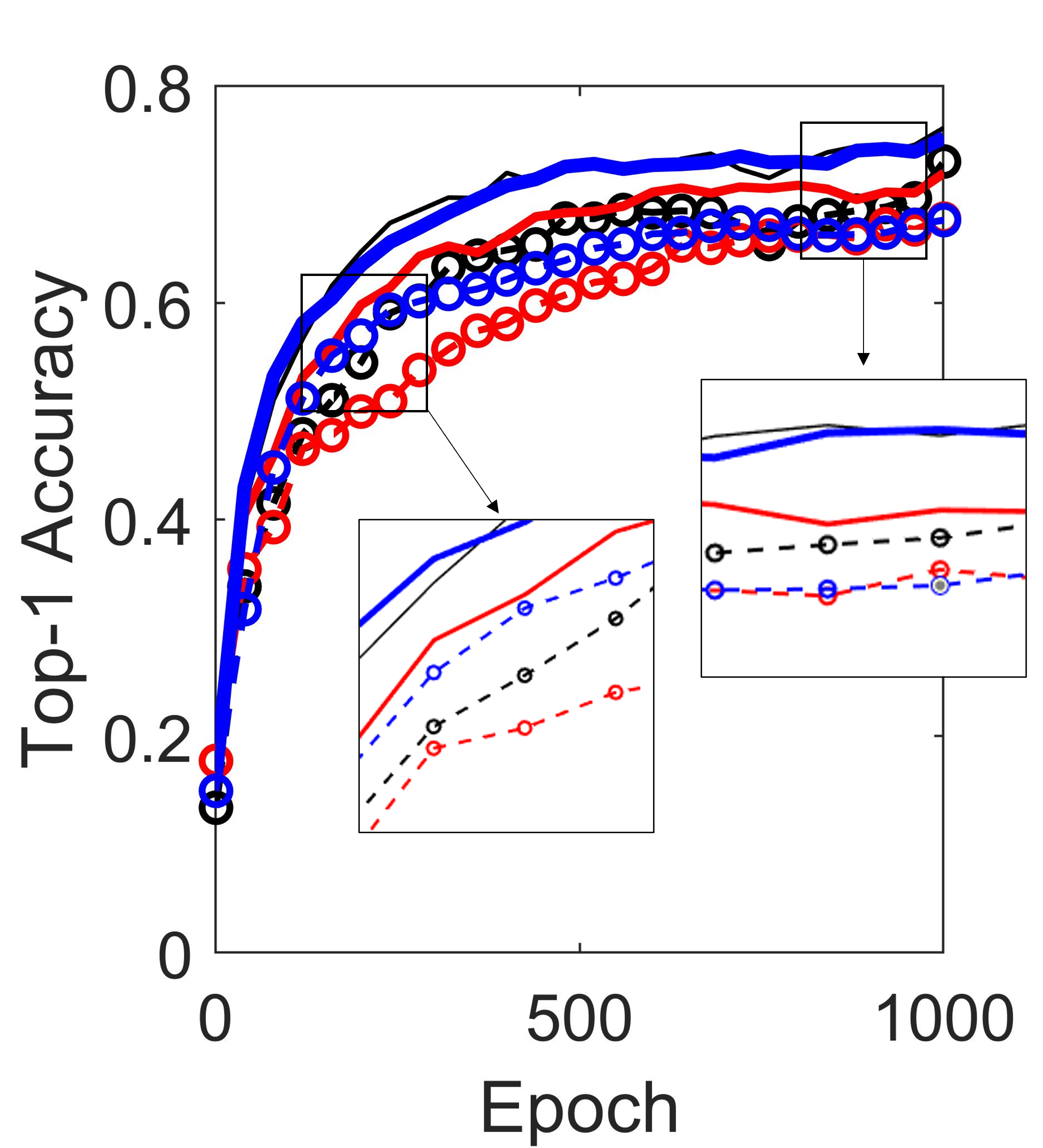} & 
     \includegraphics[width=0.45\columnwidth]{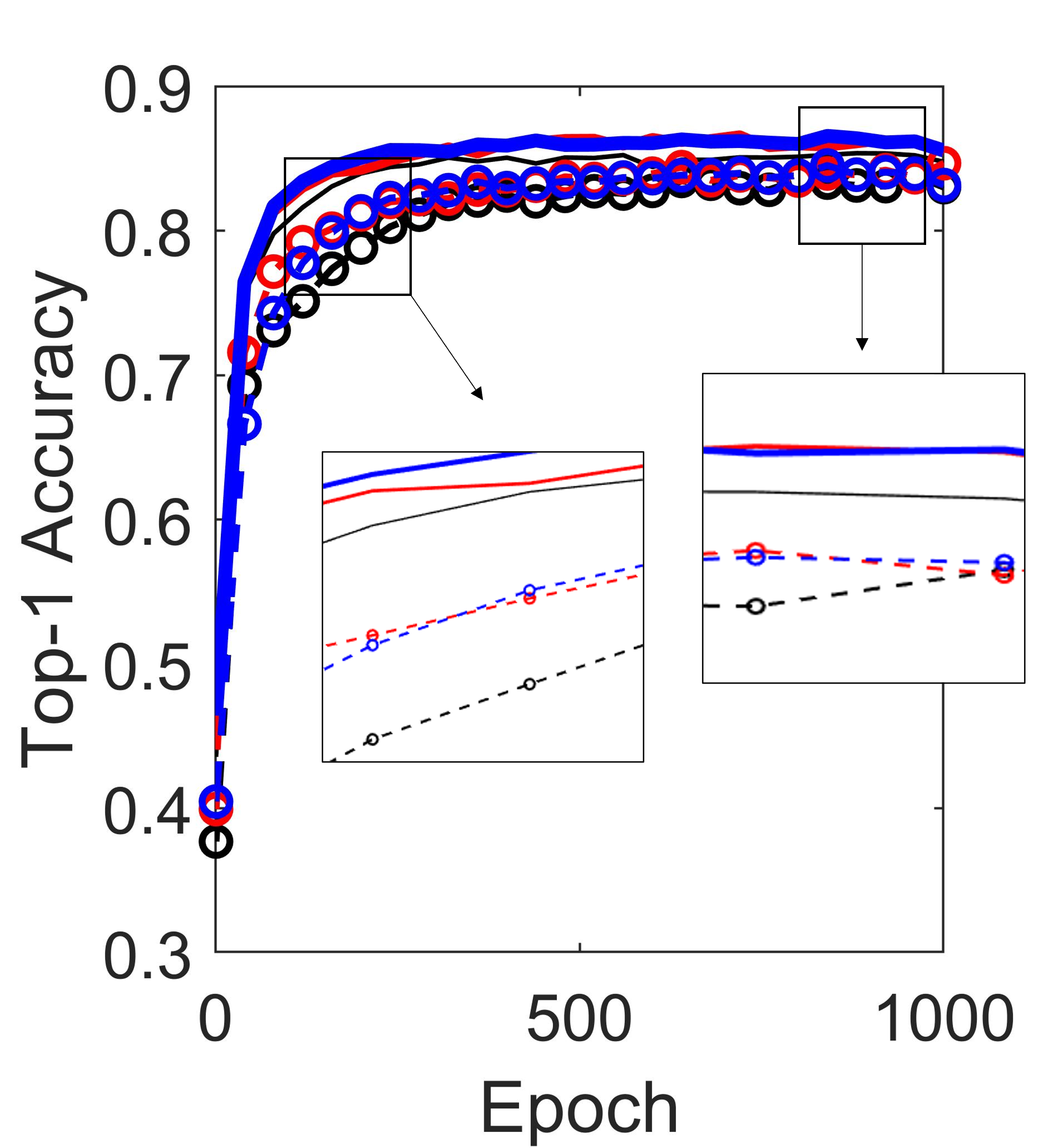}\\
     \small (a) Non-IID ($\alpha=0.1$).& \small (b) IID ($\alpha=1.0$).
\end{tabular}
\caption{Comparison of SNN training algorithms: SlimTrain \cite{yu2018slimmable}, USTrain \cite{ICCV2019_USlimmable}, and our proposed SUSTrain ($K=10$, $\sigma^2=-30\mathrm{dB}$).}
\label{fig:trainmethod}
\end{figure}

\begin{figure}[t!]
\scriptsize\centering
\begin{tabular}{cc}
    \includegraphics[width=0.48\columnwidth]{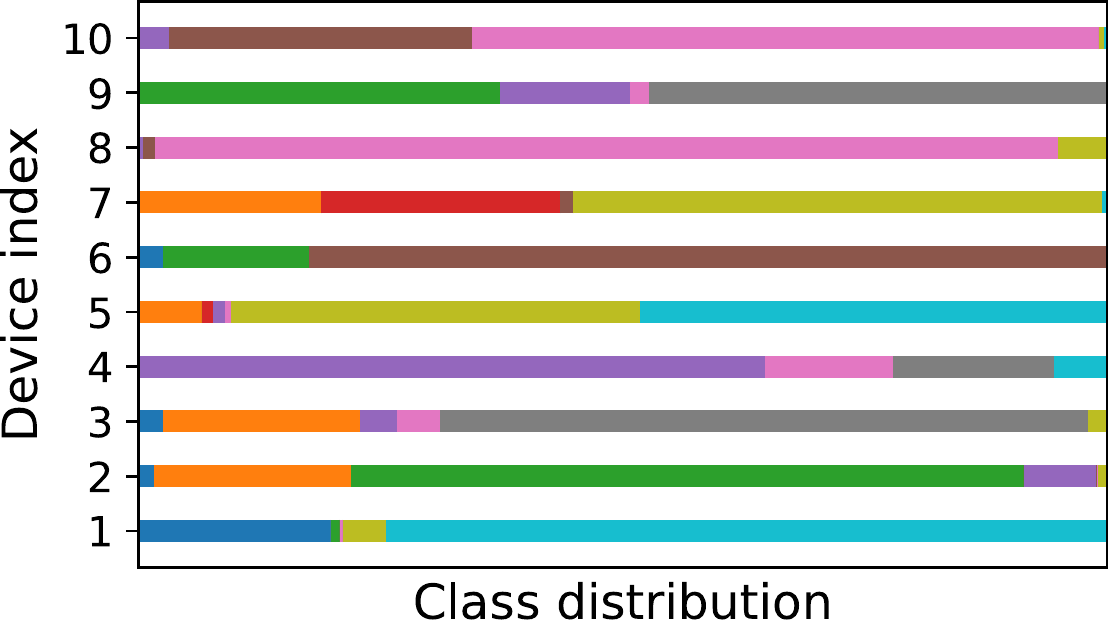} & \includegraphics[width=0.48\columnwidth]{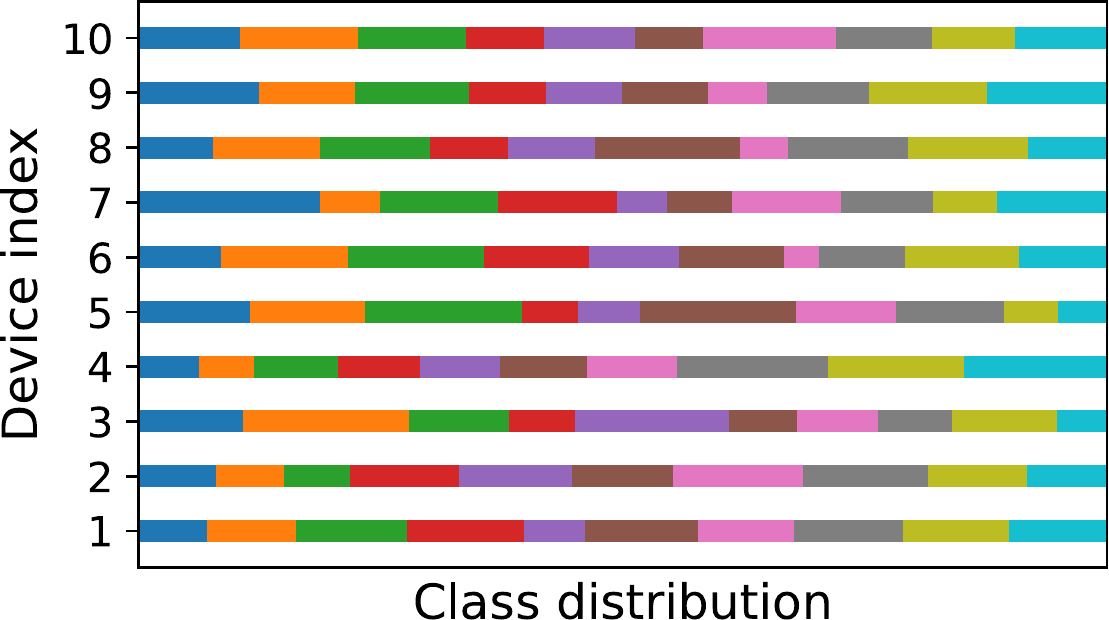}\\
    \small (a) Non-IID ($\alpha = 0.1$).&
    \small (b) IID ($\alpha = 10$). \\ 
\end{tabular}
    \caption{An illustration of the data distributions across $10$ devices for the different values of the Dirichlet concentration ratio $\alpha$.}
    \label{fig:Non-iid-distribution}
\end{figure}

To resolve this problem, we propose an ST algorithm, coined \emph{superpositioned USTrain (SUSTrain)}, which first holds all the forward propagation (FP) losses, and then concurrently updates all the width configurations with the superpositioned gradients. In doing so, a student is trained using IPKD without the logit mismatch with its full-width teacher's logit, while the teacher is simultaneously trained using the ground truth. For the device~$k$, the aforementioned local SNN update rule is,

{\small\begin{align}\label{eq:SUSTrain}
    \hspace{-5pt}\theta^k_{t+1} \!=\! \theta_t^k  \!-\! \eta_t \big[ w_1 \nabla \hat{F}^k (\theta_t^k\odot\Xi_1,\zeta^k_t) \!+\! w_2\nabla {F}^k (\theta^k_t \odot \Xi_2, \zeta^k_t)\big],
\end{align}}\normalsize
where $w_1 + w_2 = 1$ for constants $w_1, w_2>0$. The term $\eta_t>0$ is a learning rate, and $\zeta_t^k$ implies a stochastic input realization. The function $F^k(\theta_t^k\odot \Xi_i,\zeta_t^k)$ is the cross-entropy between the ground truth $y(\zeta_t^k)$ and the logit $M(\theta_t^k\odot \Xi_i,\zeta_t^k)$ of the $i$-th width configuration, whereas the IPKD function $\hat{F}^k(\theta_t^k\odot \Xi_i,\zeta_t^k)$ is the cross-entropy between the logit $M(\theta_t^k,\zeta_t^k)$ of the full-width configuration and the logit $M(\theta_t^k\odot \Xi_i,\zeta_t^k)$. Fig. \ref{fig:trainmethod} corroborates that regardless of the data distributions, SUSTrain achieves high accuracy with fast convergence, as opposed to USTrain that is effective only under IID data distributions (i.e., $\alpha=1.0$). The details of SUSTrain are in \textbf{Algorithm~\ref{alg:sustrain}}.

\begin{algorithm}[t!]\label{alg:sustrain}
\scriptsize
Initialize train parameter $\Theta =\{\theta^1,\cdots,\theta^k,\cdots,\theta^K,\theta^G\}$,\\
Initialize local dataset $\bm{Z}=\{Z_1,\cdots,Z_k,\cdots,Z_K\}$ with Dirichlet distribution\\
Initialize learning rate $\eta_t\leftarrow \eta_0$\\
Further Constraints: $\hat{F}^1(\cdot)=F^1(\cdot)$ \Comment{Discuss in Sec.~\ref{sec: convergence}}\\
\For{$t=1,\cdots,T$}
    {
    \For{$k=1,\cdots,K$}
        {
        Initialize gradients of model optimizer as $0$.\\
        Sample batch $\zeta^k_t$ from $Z_k$.\\
        Compute loss,  $loss \leftarrow F^k(\theta_t^k\odot \Xi_i,\zeta_t^k)$.\\
        Execute full-network  $M(\theta_t^k,\zeta_t^k)$.\\
        Accumulate gradients, $loss.backward()$.\\
        Execute full-network $M(\theta_t^k,\zeta_t^k)$.\\
        Stop gradients of  $M(\theta_t^k,\zeta_t^k)$ as label.\\
        \For{$i=1,\cdots,S-1$}
            {
            Execute and calculate loss $\hat{F}^k(\theta^k_t \odot \Xi_i, \zeta^k_t)$\\
            $loss \leftarrow loss  + w_i\hat{F}^k(\theta^k_t \odot \Xi_i, \zeta^k_t)$.
            }
            Calculate gradient of $loss$.\\ 
            Update model parameters.  \Comment{Eq.~(\ref{eq:SUSTrain})}
        }
    }
\caption{\small Superposition Training (SUSTrain)}
\end{algorithm}

\section{Global Model Aggregation with Superposition Coding \& Successive Decoding}\label{sec:4}

\subsection{Superposition Coding \& Successive Decoding}
At a receiver, the signal-to-interference-plus-noise ratio (SINR) is given as
$\gamma={g d^{-\beta} P}/{(\sigma^2 + P^I)}$,
where $P$, $P^I$, and $\sigma^2$ stand for the transmission, received interference, and noise powers. In addition, $d$ is a transmitter-receiver distance, $\beta\geq 2$ is a path loss exponent, and $g$ is random small-scale fading. Following the Shannon's capacity formula with a Gaussian codebook, the received throughput $R$ with the bandwidth $W$ is $R = W \log_2(1 + \gamma)$ (bits/sec).
When the transmitter encodes raw with a code rate~$u$, its receiver successfully decodes the encoded data if $R > u$. The decoding success probability is,
 
 \vspace{-3pt}{\small\begin{align}
    \Pr(R \geq u)&= \Pr(\frac{g d^{-\beta} P }{\sigma^2 + P^I} \geq u'). \label{eq:dsp1}
\end{align}}\normalsize
where $u'=2^{\frac{u}{W}}-1$. The decoding success probability with SC and SD is given by balancing $P$ and $P$ as elaborated next.

We consider simultaneously conveying $S$ messages from a transmitter to its receiver. These messages are SC-encoded before transmission \cite{Cover:TIT72}, while the total transmission power budget $P$ is allocated to the $i$-th message with the amount of $P = \sum\limits^{S}_{i=1}\nolimits P_{i}$ transmission power for $i\in[1,S]$.
When conveying only a single message, i.e., $S=1$, there exists no interference at reception, i.e., $P^I=0$. For $S>1$, SD determines the interference.

\label{sd} 
At its receiver, the SC-encoded message is supposed to be successively reconstructed by first decoding the stronger signal, followed by cancelling out the reconstructed (stronger) signal and then decoding the next stronger signal, i.e., SD, also known as successive interference cancellation~\cite{Jinho:TCOM17}. 
Under Rayleigh fading, the small-scale fading power gain $g$ follows an exponential distribution, i.e., $g\sim\textsf{Exp}(1)$.
Assuming $P_i > P_{i'}$ and $i'>i$, the receiver can successively decode the $i$-th message while experiencing the rest of the messages as its interference $P_i^I$, i.e., $P_i^I = g d^{-\beta} \hat{P}_i^I$,
where $\hat{P}_i^I\triangleq\sum^{I}_{i' = i+1}P_{i'}$ for $i\leq S-1$, and $\hat{P}^I_S=P^I_S=0$ as there is no interference for the last message. 
Let $R_i$ denote the throughput for the $i$-th message. By substituting $P^I_i$\typeout{(\ref{eq:interference})} into (\ref{eq:dsp1}), the distribution of $R_i$ is cast as, $\Pr(R_i \geq u ) = \Pr\left(g \geq \frac{c}{ P_i/u' - \hat{P}^I_i}\right)$,
where $c=\sigma^2 d^\beta$.
Applying this result, the decoding success probability $p_i$ of the $i$-th message is:

\vspace{-3pt}{\small\begin{align} 
    p_i &= \Pr(R_1 \geq u, R_{2} \geq u,\cdots, R_i \geq u  )\\
    &= \Pr \!\left( g\geq \frac{c}{ P_1/u' - \hat{P}^I_1}, \cdots , g \geq \frac{c}{ P_i/u' - \hat{P}^I_i} \right) \\
    &= \Pr\!\left( g \!\geq\!  \max \!\left\{\!
    \frac{c}{ P_1/u' - \hat{P}^I_1} , \cdots\!, \frac{c}{ P_i/u' - \hat{P}^I_i} 
    \!\right\}\!\right)\!. \label{eq:sicdsp1}
\end{align}}\normalsize

\begin{algorithm}[t!]
\scriptsize 
Initialize train parameters $\Theta = \{\theta^1,\cdots,\theta^k,\cdots,\theta^K,\theta^G\}$.\\
Split dataset $\bm{Z}$ into $K$ datasets $\bm{Z}=\{Z_1,\cdots, Z_k, \cdots, Z_K\}$.\\
\While{Training}{
    \texttt{//Local Model Training (Algorithm 1)}\\
    \For{$n=1,\cdots,K$}
        {
        \For{$\zeta_k$ in $Z_k$}
            {
            Update local model parameter $\theta^k$ \Comment{Eq.~(\ref{eq:SUSTrain})}
            }
        }

    \texttt{//SC\&SD-based Server Aggregation (Uplink)}\\
    \If{Aggregation Period}{
        $n_\mathsf{L}=|\mathsf{H\cup}\mathsf{F}| \leftarrow 0, n_\mathsf{R}=|\mathsf{F}| \leftarrow 0, \mathsf{H} \leftarrow \varnothing, \mathsf{F} \leftarrow  \varnothing$\\
        \For{$k=1,\cdots,K$}
            {$g \leftarrow rand(1)$\\
            \If{$p_2 \leq g < p_1$}
                {
                $ \mathsf{H} \leftarrow \mathsf{H} \cup k$,~$n_\mathsf{L} \leftarrow n_\mathsf{L} + 1$
                }            
            \If{$g\geq p_2$}
                {
                $ \mathsf{F} \leftarrow \mathsf{F} \cup k$,~$n_\mathsf{L} \leftarrow n_\mathsf{L} + 1$,~$n_\mathsf{R} \leftarrow n_\mathsf{R} + 1$
                }
            }
            $\triangleright$ \bf{Case1.} $n_\mathsf{L}>0$, $n_\mathsf{R}>0$\\
            \hspace{10pt}$\theta^G \leftarrow \frac{1}{|\mathsf{H}\cup\mathsf{F}| }{\sum_{k\in\mathsf{H}\cup\mathsf{F}}\theta^k\odot\Xi} + \frac{1}{|\mathsf{F}|}{\sum_{k\in\mathsf{F}}\theta^k\odot\Xi^{-1}} $ \\
            $\triangleright$ \bf{Case2.} $n_\mathsf{L}>0$, $n_\mathsf{R}=0$\\
            \hspace{10pt}$\theta^G \leftarrow \frac{1}{n_\mathsf{L}}\sum_{k\in\mathsf{H}}(\theta^k\odot\Xi)$\\
            $\triangleright$ \bf{Case3.} $n_\mathsf{L}=n_\mathsf{R}=0$\\
            \hspace{10pt} {Skip aggregation }
        }
    \texttt{//Local Update (Downlink)}\\
        \For{$n=1,\cdots,K$}
            {
                $\theta^k \leftarrow \theta^G$
            }
}
\caption{\small SlimFL with SC \& SD}
\label{alg:SlimFL}
\end{algorithm}

\subsection{SlimFL Operations}
We elaborate the SlimFL and global model aggregation. The notation for SlimFL is as summarized in Tab.~\ref{tab:notation-convergence}. 
The overall SlimFL operations are described in \textbf{Algorithm~\ref{alg:SlimFL}}.
The network consists of $K$ devices connected to a parameter server over wireless links. In the uplink from each device to the server, the device applies SC, and the server utilizes SD. To be precise, for every $k\in [1,K]$, the $k$-th local device has its local dataset $Z^k \in \bm{Z}$ and its SNN parameter $\theta^k$ with $2$ width configurations. The global data $\bm{Z}$ can be either IID or non-IID across devices. Each SNN $\theta^k$ is divided into the LH segment $\theta^k\odot\Xi$ and the RH segment $\theta^k\odot\Xi^{-1}$, where $\Xi=\Xi_1$ and $\Xi^{-1}=\Xi_2 - \Xi_1$.
The $k$-th local device is trained with superposition training (lines 4--9), which is written as (\ref{eq:SUSTrain}).
The local device uploads the SC-encoded local model $\theta^k$ to the server. All local devices transmit two messages (i.e., LH and RH segments) with different transmission power $P_1$ and $P_2$ where $P_1\gg P_2$. 
After reception, according to (\ref{eq:sicdsp1}), the server can successively decode using SD, and obtain: (i) $0.5$x model if {\small$g\geq  c/(P_1/u' - P_2)$\normalsize} (lines 15--17); (ii) $1.0$x model if the channel fading gain satisfies {\small$g\geq  \max\{ c/(P_1/u' - P_2), c/(P_2/u') \}$\normalsize} (lines 18--20); and (iii) otherwise it obtains no model. Accordingly, the server aggregates the RH segments from $\mathsf{F}$ of devices, and the LH segments from $\mathsf{H}\cup \mathsf{F}$ of devices.

Hereafter, for the convergence analysis in the next section, we assume that $K$ is sufficiently large such that $|\mathsf{H} \cup \mathsf{F}| \approx K p_1 $ and $|\mathsf{F}| \approx K p_2$, where $p_1$ and $p_2$ are the decoding success probabilities of the LH and RH segments, respectively, given in (\ref{eq:sicdsp1}). Consequently, at the $t$-th communication round, the server constructs a global model $\theta^{G}_{t}$ as follows:

\vspace{-3pt}{\small\begin{align}
\theta^{G}_{t} \leftarrow \frac{1}{K p_1}\sum_{k \in \mathsf{H} \cup \mathsf{F}}\nolimits \theta^{k}_{t} \odot \Xi + \frac{1}{K p_2}\sum_{k \in \mathsf{F}}\nolimits\theta^{k}_{t} \odot \Xi^{-1}.
\label{eq:global}
\end{align}}\normalsize

Although SlimFL is flexible enough to incorporate various training and communication techniques, henceforth we limit our scope by considering the following assumptions.
\begin{itemize}
    \item The downlink decoding is always successful (lines 29--32), ignoring SC and SD. This is partly advocated by the fact that the server (e.g., a base station) has much larger large transmit power than the uplink power.
    
    \item The number of local iterations per communication round is $1$, therefore omitting the superscript $G$, i.e., $\theta_t = \theta_t^G$.
    
\end{itemize}
These assumptions make the analysis of SlimFL mathematically amenable, as we shall elaborate in the next section.

\section{SlimFL Convergence Analysis}\label{sec: convergence}

To show the convergence of SlimFL, we follow the key derivation techniques utilized in \cite{li2019convergence,khaled2020tighter} for FedAvg. Nonetheless, SlimFL convergence analysis is non-trivial. One major reason is that the local model updates in (\ref{eq:SUSTrain}) and the global model aggregation in (\ref{eq:global}) include complicated binary masks due to the SNN architecture as well as SC and SD. Therefore, as opposed to FedAvg whose global objective function to be minimized is the weighted average of local loss functions $\{F^k(\theta_t^k)\}$, i.e., empirical risk, the objective function $F(\theta_t)$ of SlimFL is unclear. Alternatively, we define $F(\theta_t)$ based on its gradient $f_t = \nabla F(\theta_t)$ that can be derived through the local and global operations of SlimFL as detailed next.

After the downlink, the device $k$ replaces its local model with the downloaded global model, i.e., $\theta^{k}_{t} \leftarrow \theta_{t}$. Then, the device updates the local model, yielding:
     
     \vspace{-3pt}{\small\begin{align}
     \theta^{k}_{t+1} \leftarrow \theta_{t} -\eta_{t}g^{k}_{t},
     \label{eq:22}
     \end{align}}\normalsize
     where $g^{k}_{t} = \sum^{2}_{i = 1} \omega_{i}\nabla F^{k}(\theta_{t} \odot \Xi_{i}, \zeta^{k}_{t})$ follows from (\ref{eq:SUSTrain}). For mathematical tractability, here we assume that the soft target of the student can be approximated as the hard target, i.e., $ \hat{F}^k(\theta_{t} \odot \Xi_{i}, \zeta^{k}_{t}) \approx F^{k}(\theta_{t} \odot \Xi_{i}, \zeta^{k}_{t})$. 
     
     Next, after the uplink, the server  aggregates the updated local models, constructing the global model $\theta_{t+1}$. Applying (\ref{eq:22}) to (\ref{eq:global}), the constructed global model is cast as:
     
{\small\begin{align}
\theta_{t+1} 
&\!=\! \frac{1}{Kp_{1}}\!\!\sum_{k \in \mathsf{H} \cup \mathsf{F}}(\theta_{t} \!-\!\eta_{t}g^{k}_{t})\!\odot\!\Xi \!+\!\frac{1}{Kp_{2}}\!\!\sum_{k\in\mathsf{F}}(\theta_{t} \!-\!\eta_{t}g^{k}_{t})\!\odot\!\Xi^{-1}\hspace{-15pt}\\
&=\theta_{t}-\eta_{t}\Big(\underbrace{\frac{1}{Kp_{1}}\sum_{k \in \mathsf{H}\cup\mathsf{F}}g^{k}_{t}\odot\Xi + \frac{1}{Kp_{2}}\sum_{k \in \mathsf{F}}g^{k}_{t}\odot\Xi^{-1}}_{:=f_t}\Big),
\label{eqn:globalgrad}
\end{align}}\normalsize
resulting in $f_t$ in (\ref{eqn:globalgrad}), which characterizes $F(\theta_t)$. In (\ref{eqn:globalgrad}), the last step can be obtained from $|\mathsf{H}\cup\mathsf{F}|=Kp_1$, $|\mathsf{F}|=Kp_2$, and $\theta_t = \theta_t \odot (\Xi +  \Xi^{-1})$.

Hereafter we use the bar notation $\bar{\cdot}$ for the value averaged over $\{\zeta_t^k\}$, and $^*$ for indicating the optimum. For the functions $F$ and $\{F^k\}$, we consider the following assumptions that are widely used in the literature \cite{li2019convergence,stich2018local}.

\begin{assumption}\label{asm:L-smoothness}\textbf{(L-smoothness)}
	$F$ and $\{F^k\}$ are $L$-smooth, i.e., 
	$F^k(\theta_{v})  \leq F^k(\theta_{w}) + (\theta_{v} - \theta_{w})^T \nabla F^k(\theta_{w}) + \frac{L}{2} \| \theta_{v} - \theta_{w}\|^2$ for all $v,w>0$.
\end{assumption}
\begin{assumption}\label{asm:mu-strconvex}\textbf{($\bm{\mu}$-strong convexity)}
	$F$ and $\{F^k\}$ are $\mu$-strong convex:
	i.e., $F^k(\theta_{v})  \geq F^k(\theta_{w}) + (\theta_{v} - \theta_{w})^T \nabla F^k(\theta_{w}) + \frac{\mu}{2} \| \theta_{v} - \theta_{w}\|^2$ for all $v, w>0$. 
\end{assumption}
\begin{assumption}\label{asm:noniid-grad-bounded}\textbf{(Bounded local gradient variance)}
The variance of the local gradient $\nabla F^k(\theta^k,\zeta^k_t)$ is bounded within $Z_k$, which is given as 
	    $\mathbb{E}[\|\nabla F^k(\theta^k,\zeta^k_t)-\nabla \bar{F}^k(\theta)\|^2] \leq \sigma_k^2$.
\end{assumption}

Inspired by \cite{khaled2020tighter}, we define $\delta = \frac{1}{K}\sum_{k=1}^K \sigma_k^2$ as a factor that measures the non-IIDness of $\bm{Z}$. Indeed, $\frac{1}{K}\sum_{k=1}^K (\sigma_k - \frac{1}{K}\sum_{k=1}^K \sigma_k )^2$ is the variance (over $k$) of the local gradient variance (over $Z_k$). This characterizes the data distributions over devices, and so does $\delta$ without loss of generality.

To prove the convergence of SlimFL, we derive the following two lemmas.
\begin{lemma}\textbf{(Bounded global gradient variance)} \label{lem--1}
Under Assumption~\ref{asm:noniid-grad-bounded}, the variance of the global gradient $f_t$ is bounded within $\bm{Z}$, which is given as $\EB\|{f}_t-\bar{f}_t\|^2\leq B$ where $B = 4\delta (\frac{1}{p_1}+\frac{1}{p_2})\sum^{2}_{i=1}w^2_{i}$.
\end{lemma}
\begin{proof}
According to $f_t$ in (\ref{eqn:globalgrad}) and Assumption~\ref{asm:noniid-grad-bounded}, $\|f_t-\bar{f}_t\|^2 = \|\frac{1}{Kp_1}\sum_{{k}\in\mathsf{H}\cup\mathsf{F}}(g^k_t-\bar{g}^k_t)\odot\Xi+\frac{1}{Kp_2}\sum_{{k}\in\mathsf{F}}(g^k_t - \bar{g}^k_t) \odot \Xi^{-1}\|^2
\leq \frac{2}{Kp_1}\sum_{k\in\mathsf{H}}\|(g^k_t-\bar{g}^k_t)\odot\Xi\|^2 + \frac{2}{Kp_2}\sum_{k\in\mathsf{F}}\|(g^k_t-\bar{g}^k_t)
\odot\Xi^{-1}\|^2\leq \frac{2}{Kp_1}\sum_{k \in \mathsf{H}}\|g^k_t-\bar{g}^k_t\|^2 + \frac{2}{Kp_2}\sum_{k\in\mathsf{F}}\|g^k_t-\bar{g}^k_t\|^2$, where the first inequality follows from the Cauchy–Schwarz (C-S) inequality, and the last step is because $\|X\odot\Xi\|^2 \leq \|X\|^2$. Similarly, $ \|g^k_t-\bar{g}^k_t\|^2 = \|\sum^2_{i=1}\nolimits w_i(\nabla F^k(\theta_t,\zeta^k_t)-\nabla F^k(\theta_t))\odot \Xi_i\|^2 
\leq 2 \sum^2_{i=1}\nolimits w^2_i\|\nabla F^k(\theta_t,\zeta^k_t) - \nabla F^k(\theta_t) \|^2$ and taking expectation of both sides gives 
$\EB\|g^k_t-\bar{g}^k_t\|^2\leq 2\sigma^2_k\sum^2_{i=1}\nolimits w^2_i.$ Combining these results finalizes the proof.
\end{proof}

\begin{lemma}\label{lem---2}\textbf{(Per-round global model progress)} Under  Assumptions~\ref{asm:L-smoothness} and \ref{asm:mu-strconvex} with a learning rate $\eta_t \leq \frac{1}{L}$ , the error between the updated global model and its optimum progresses as 
$\EB\|\theta_{t+1}-\theta^*\|^2\leq (1-\mu\eta_t)\EB\|\theta_t-\theta^*\|^2+ \eta^2_t{B}$.
\end{lemma}
\begin{proof}
According to (\ref{eqn:globalgrad}), we have\\
$\|\theta_{t+1}-\theta^*\|^2=\|\theta_t-\eta_t f_t - \theta^* - \eta_t \bar{f}_t + \eta_t\bar{f}_t\|^2=$ \resizebox{1\columnwidth}{!}{$\underbrace{\|\theta_t-\theta^*-\eta_t\bar{f}_t\|^2}_{A_1}+\underbrace{2\eta_t\langle\theta_t-\theta^*-\eta_t f_t,\bar{f}_t-f_t\rangle}_{A_2} 
    + \underbrace{\eta^2_t\|f_t-\bar{f}_t\|^2}_{A_3}=$} $\|\theta_t-\theta^*_t\|^2-\underbrace{2\eta_t\langle\theta_t-\theta^*,\bar{f}_t\rangle}_{B_1} + \underbrace{\eta^2_t\|\bar{f}_t\|^2}_{B_2} + A_2 + A_3$.
Here, $\mathbb{E}[A_2] = 0$ due to $\mathbb{E}(f_t) = \bar{f}_t$, and $A_3$ is bounded according to Lemma~\ref{lem--1}. 
Note that $\bar{f}_t = \EB[f_t] = \EB[\nabla F(\theta_t)] = \nabla \EB[F(\theta_t)]$, and $\EB[F]$ inherits the $\mu$-strong convexity and L-smoothness from $F$.
By the L-smoothness of $\EB[F]$, we have
$\|\bar{f}_t\|^2 \leq 2L(\EB[F(\theta_t)-{F}(\theta^*)])$, showing the boundness of $B_2$.
Next, by the $\mu$-strong convexity of $\EB[F]$, we have
$\langle \theta^* - \theta_t, \bar{f}_t \rangle \leq \EB[F(\theta^*)-F(\theta_t)] - \frac{\mu}{2}\|\theta_t-\theta^*\|^2$, proving the boundness of $B_1$.
Applying the bounds of $B_1$ and $B_2$, we obtain $A_1 \leq (1- {\mu \eta_{t}})\|\theta_{t} -\theta^{*}\|^{2} -2\eta_{t}(1-L\eta_t)\EB[F(\theta_{t}) -F(\theta^{*})]$, where the last term on the RHS vanishes for $\eta_t < \frac{1}{L}$. Taking the expectation at both sides completes the proof.
\end{proof}

Now we are ready to prove our main theorem.
\begin{theorem} (\textbf{SlimFL Convergence})\label{convth2}
Under Assumptions \ref{asm:L-smoothness}--\ref{asm:noniid-grad-bounded} with the learning rate $\eta_t = \frac{2}{\mu{t}+2L-\mu}$, one has

\vspace{-3pt}{\small\begin{align}
\EB[F(\theta_{t})] - F^{*} 
\leq \frac{L}{\mu}\cdot\frac{\mu L \Delta_1 + 2B }{\mu t+ 2L-\mu},\label{eq:thm1}
\end{align}}\normalsize
where $B = 4\delta (\frac{1}{p_1}+\frac{1}{p_2})\sum^{2}_{i=1}w^2_{i}$ and $\Delta_t \triangleq \EB\|\theta_t-\theta^*\|^2$. Therefore, $\EB[F(\theta_{t})]$ converges to $F^{*}$ as $t\rightarrow\infty$.
\end{theorem}
\begin{proof}
Since $\eta_t = \frac{2}{\mu{t}+2L-\mu} \leq \frac{1}{L}$, applying Lemma~\ref{lem---2}, we have
    $\Delta_{t+1} \leq \left(1 - \mu\eta_{t}\right)\Delta_{t} + \eta_{t}^{2}{B}$.
By induction, we aim to show that $\Delta_t \leq \frac{v}{t + 2\kappa-1}$ where $\kappa = \frac{L}{\mu}$ and $v = \max\{2\kappa\Delta_{1},{4B}/{\mu^2}\}$ as elaborated next.
By the definition of $v$, it is trivial that $\Delta_1\leq\frac{v}{2\kappa}$. Assuming that $\Delta_{t'} \leq \frac{v}{t' + 2\kappa-1}$ holds, we have
$\Delta_{t'+1} \leq (1-\mu\eta_{t'})\Delta_{t'} + \eta^{2}_{t'}B \leq \left(1- \frac{2}{t'+2\kappa -1}\right) \frac{v}{t'+2\kappa-1}+\frac{{4B}/{\mu^2}}{(t'+2\kappa-1)^2} = \frac{(t'+2\kappa-2)v-(v-{4B}/{\mu^{2}})}{(t'+2\kappa-1)^{2}} \leq \frac{t'+2\kappa-2}{(t'+2\kappa-1)^{2}}v \leq \frac{v}{t'+2\kappa}$,
which proves that $\Delta_{t} \leq \frac{v}{t + 2\kappa-1}$. For $t=1$, we obtain
    $v = \max\{2\kappa\Delta_{1},\frac{4B}{\mu^2}\} \leq 2\kappa\Delta_{1} + \frac{4B}{\mu^2}$. Finally, 
by the L-Smoothness of $F$, one has
    $\EB[F(\theta_{t})] - F^{*}  =\EB[F(\theta_{t}) -F(\theta^{*})] \leq \frac{L}{2}\EB\|\theta_{t} -\theta^{*}\|^{2}$. Applying Lemma~\ref{lem---2} with the aforementioned results, we have 
    $\EB\|\theta_{t} -\theta^{*}\|^{2} \leq \frac{v}{t +2\kappa -1} \leq \frac{2}{\mu}\cdot\frac{\mu L \Delta_1 + 2B}{\mu t+ 2L-\mu}$, which completes the proof of the theorem.
\end{proof}
The result of Theorem 1 exhibits several insightful characteristics of SlimFL as follows.

\BfPara{Robustness to poor channels}
In (\ref{eq:thm1}), we observe that aggregating more 0.5x and 1.0x models (i.e., increasing $p_1$ and $p_2$) equally contributes to reducing the global optimality gap. Therefore, aggregating 0.5x models can complement the frequent decoding failures of 1.0x models under poor channels.

 \BfPara{Failure under extremely poor channels} Consider an extremely poor channels where the server is unable to decode 1.0x models while aggregating only 0.5x models (i.e., $p_2\approx 0$ and $p_1>0$). In this case, the optimality gap diverges although it aggregates 0.5x models. Under such channel conditions, SC becomes useless, and vanilla FL with only 0.5x models is preferable to SlimFL.

 \BfPara{Robuestness to non-IID data} The optimality gap increases with $\delta$ (i.e., more non-IID). The increased gap can be counteracted by aggregating not only 1.0x models but also 0.5x models, as opposed to vanilla FL that benefits only from aggregating either 0.5x models or 1.0x models. 

Judging from the aforementioned observations, we conclude that SlimFL is preferable for non-IID data distributions and moderately poor channel conditions where $0 \ll p_1, p_2 < 1$. For an extremely good (i.e., $p_2\approx 1$) or an extremely poor (i.e., $p_1\approx 0$) channel conditions, vanilla FL with only 1.0x models or 0.5x models is preferable, respectively. These favorable conditions and effectiveness of SlimFL will be corroborated by simulation in Sec.~\ref{sec:5}.

Furthermore, Theorem 1 provides the design guidelines on SC and ST as elaborated in the following two propositions.
\begin{proposition}[\textbf{Optimal SC power allocations}]\label{remark:1}
Consider the SC power allocation ratio $\lambda \in (0.5,1]$ such that $P_1=\lambda P$ and $P_2 = (1-\lambda)P$. If $\lambda \gg \max\left\{ 0.5, {c u'(1+u')}/{P} \right\}$, 
the optimal SC power allocation ratio that minimizes the RHS of (\ref{eq:thm1}) is given as $\lambda^* = \frac{u'+\sqrt{1+u'} -1}{u'}$.
\begin{proof}
Define $D\triangleq \frac{1}{p_1}+\frac{1}{p_2}$. According to the RHS of (\ref{eq:thm1}), $\lambda^*$ minimize $D$. Since $P_1>P_2$, we have $D = \exp\left(\frac{c}{\lambda P/u'-(1-\lambda )P}\right)+\exp\left(\frac{c}{ (1-\lambda)P/u'}\right)$. If $\lambda \gg c u'(1+u')/{P}$, 
we can approximate the both terms in $D$ using the first-order Taylor expansion, yielding $D\approx 2+ \frac{c}{\lambda P/u'-(1-\lambda )P}+\frac{c}{ (1-\lambda)P/u'}$. The approximated $D$ is convex, and the optimum is given by the first order necessary condition.
\end{proof}

\end{proposition}
 Note the condition $\lambda \gg \max\left\{ 0.5, {c u'(1+u')}/{P} \right\}$ above can be satisfied under small model sizes (e.g., $t'\rightarrow 0$), large bandwidth (e.g., $W\rightarrow \infty$), good channel conditions (e.g., $\sigma^2\rightarrow 0$), and/or large total transmit power budget (e.g., $P\rightarrow \infty$). For practical scenarios, by simulation we confirm that the analytic optimum is indistinguishable from the numerical optimum as shown in Fig~\ref{fig:allocation}.

\begin{proposition}[\textbf{Optimal ST weights}]\label{prop:superposition}
The optimal ST weights that minimize the RHS of (\ref{eq:thm1}) are given as $w_1^*=w_2^*=1/2$.
\begin{proof}
The RHS of (\ref{eq:thm1}) is minimized at the minimum of $\sum_{i=1}^2 w_i^2$. By the C-S inequality, we have $\sum^2_{i=1}w_i^2 \geq \frac{1}{2}(\sum^2_{i=1} w_i)^2$. By the condition $\sum^2_{i=1}w_i = 1$ in (\ref{eq:SUSTrain}) and the equality condition of the AM-GM inequality, one has the desired result.
\end{proof}
\end{proposition}
 The impact of $\lambda^*$ and $w_i^*$ will be shown by simulation in Fig.~\ref{fig:opt} in the next section.

\section{Experiments}\label{sec:5}
 To show the effectiveness and feasibility of SlimFL, we present the performance of SlimFL exploiting SC and SD compared to its Vanilla FL counterpart without SC nor SD, in terms of accuracy, communication efficiency, and energy efficiency, as well as their robustness to various channel conditions and non-IID data distributions.
\subsection{Experimental Setup}

\BfPara{Baselines}
Our goal is enabling each device to obtain both large and small models so as to cope with its large and small energy levels in future. To this end, by leveraging SNNs with SC and SD, SlimFL simultaneously exchanges and trains $0.5$x and $1.0$x models by consuming the per-device bandwidth $W$, uplink transmission power $P$. This is compared with a Vanilla FL baseline, \emph{Vanilla FL-$1.5$x}. Due to the lack of width-adjustable SNNs, each device in Vanilla FL-$1.5$x separately runs fixed-width $0.5$x and $1.0$x models, referred to as \emph{Vanilla FL-$0.5$x} and \emph{Vanilla FL-$1.0$x}, respectively. Without SC nor SD, the device exchanges both $0.5$x and $1.0$x models separately. In brief, Vanilla FL-$1.5$x is tantamount to simultaneously running the two federated averaging operations separately for $0.5$x and $1.0$x models by doubling the bandwidth, transmission power, and computing resources. For clarity, we report the performance of Vanilla FL-$0.5$x and Vanilla FL-$1.0$x individually if available (i.e., accuracy, received bits), and otherwise we report only Vanilla FL-$1.5$x (i.e., energy cost).

\begin{table}[t!]
    \caption{Simulation Parameters.}
    \label{tab:tab_parameters}
    \scriptsize \centering
    \resizebox{0.6\columnwidth}{!}{\begin{tabular}{l|r}
    \toprule[1pt]
      \bf{Description}                & \bf{Value}  \\ \midrule
        Initial learning rate ($\eta_0$)        & $10^{-3}$ \\
        Optimizer                     & Adam\\ 
        Distance  ($d$)               &   100 [m]\\
        Path loss exponent ($\beta$)  & 2.5 \\
        Bandwidth per device ($W$)& $75\times10^6$ [Hz] \\
        Central frequency ($f_c$) &  $5.9$ [GHz] \\
        Uplink transmission power  ($P$) & $23$ [dBm] \\
        Noise power spectrum ($N_0$) &  $-169~\mathrm{[dB/Hz]}$\\
        \bottomrule[1pt] 
    \end{tabular}}
\end{table}
\begin{figure}[t!]\centering
    \begin{tabular}{c}
        \includegraphics[width=0.7\columnwidth]{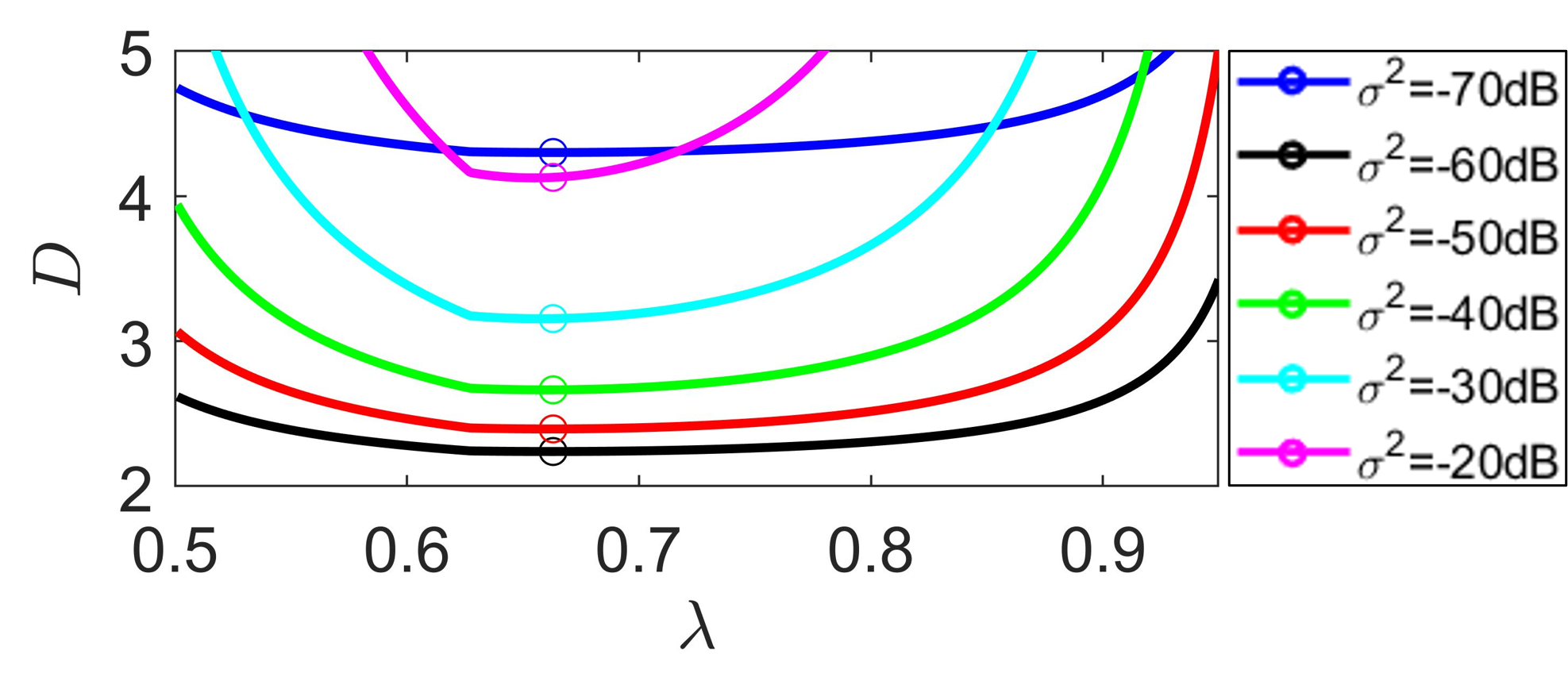}
    \end{tabular}
    \caption{SC power allocation ratio $\lambda$ versus $D$ ($:=1/p_1+1/p_2$).}
    \label{fig:allocation}
\end{figure}

\BfPara{Simulation Settings}
We consider a classification task by default with the Fashion MNIST dataset. Following the method proposed in~\cite{Arxiv_2019_NonIID}, the non-IIDness of the dataset distribution across devices is controlled by the Dirichlet distribution with its concentration parameter $\alpha\in\{0.1, 1.0, 10\}$, where a lower $\alpha$ is more non-IID distributed (i.e., more imbalanced numbers of samples over labels across devices), as visualized in Fig.~\ref{fig:Non-iid-distribution}. 
A single round of uplink and downlink communications is followed by every single local training epoch. {The communication channels over different devices are orthogonal in both uplink and downlink.} The small-scale fading gain $g$ for each channel realization follows an exponential distribution $g\sim\textsf{Exp}(1)$, i.e., Rayleigh fading \cite{TseBook:FundamaentalsWC:2005}. Communication hyperparameters are summarized in Tab.~\ref{tab:tab_parameters}. 

\subsection{Guidelines for SlimFL}
By convergence analysis, we have guidelines in respect to optimal power allocation (Proposition~\ref{remark:1}) and determination of weight parameters in superposition learning (Proposition~\ref{prop:superposition}).

\begin{figure}[t!]
\centering
\begin{tabular}{cc}
     \multicolumn{2}{c}{\includegraphics[width=0.96\columnwidth]{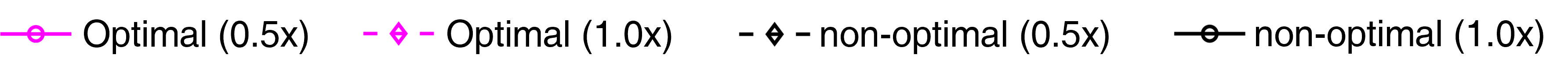}}\\
    \includegraphics[width=0.46\columnwidth]{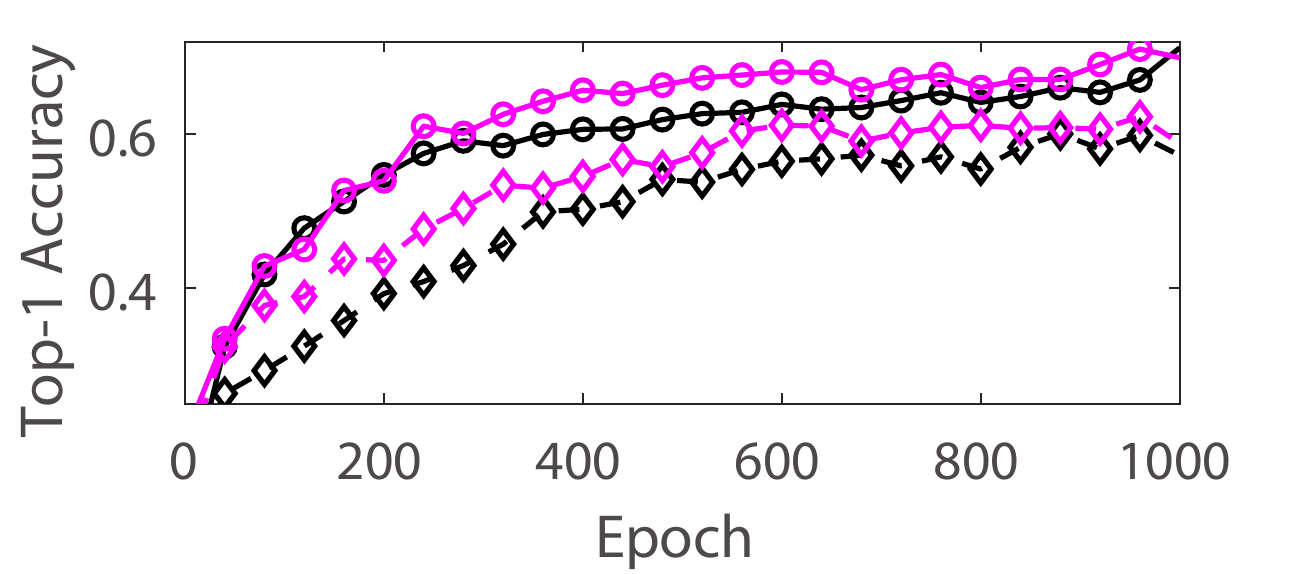}\label{lambda}& \includegraphics[width=0.46\columnwidth]{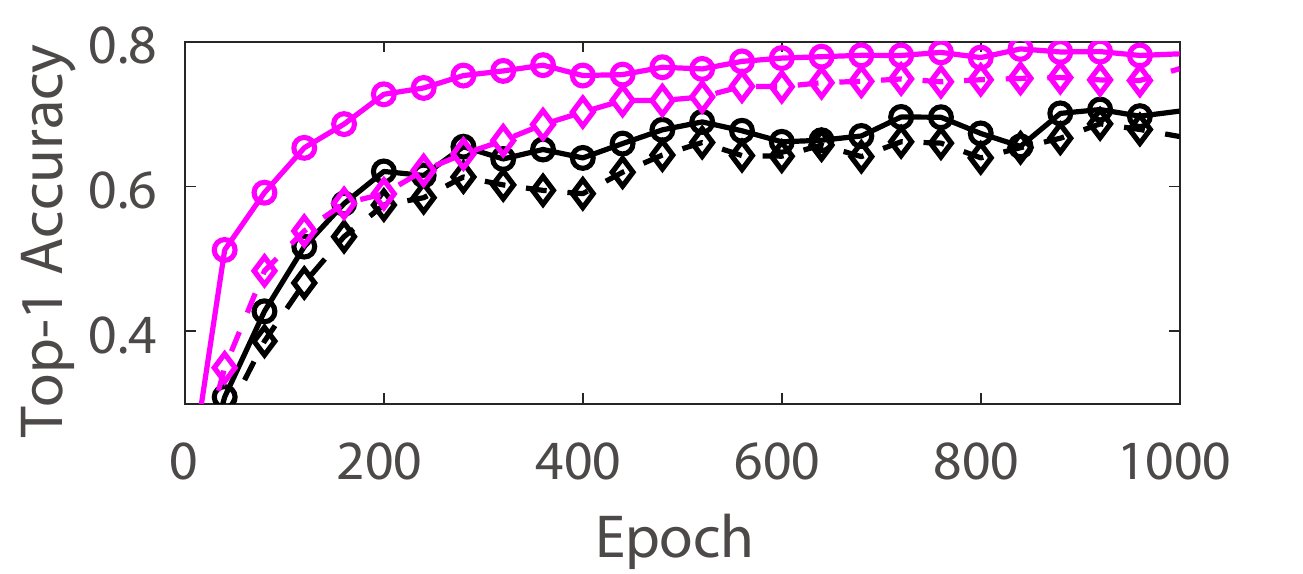}\label{wl} \\
    \small (a) SC power allocation ratio ($\lambda$). & \small (b) ST weight ($w_i$)\\
\end{tabular}
   
\caption{Top-1 accuracy under optimal and non-optimal design parameters: (a) $\lambda^*=0.663$ and $\lambda=0.8$, and (b) $w^*_1=w^*_2 = 0.5$ and $w_1 = 0.3$, and $w_2 = 0.7$) with $\alpha=0.1$.}
\label{fig:opt}
    \vspace{-1mm}
\end{figure}

\begin{figure}[t!]
\centering
\begin{tabular}{cc}
    \multicolumn{2}{c}{\includegraphics[width=0.46\columnwidth]{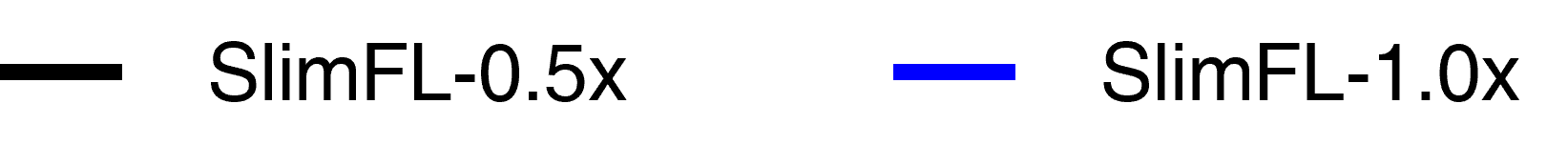}}\\
    \includegraphics[width=0.48\columnwidth]{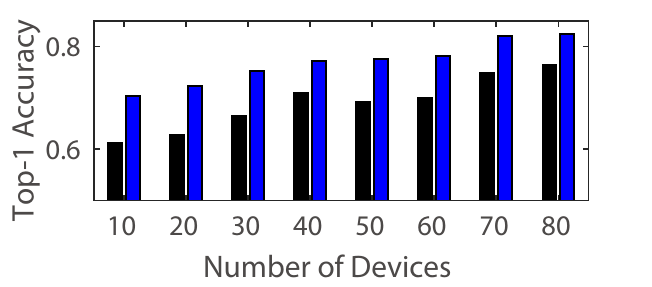}\label{fig:Number of worker}& \includegraphics[width=0.48\columnwidth]{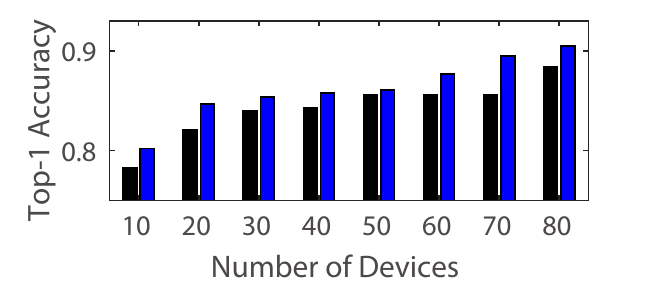}\label{fig:Number of worker2} \\
    \small (a) Top-1 accuracy ($\alpha=0.1$) & \small (b) Top-1 accuracy ($\alpha=1.0$)\\
\end{tabular}
\caption{Top-1 accuracy with different number of devices.}
\end{figure}

\BfPara{Optimal SC Power Allocations}\label{guide:powerallocation}
We propose a method of adjusting the parameter $\omega_l$ of ST to optimize our proposed model. Fig.~\ref{fig:allocation} represents the numerical results of optimal $\lambda^*$. In Proposition 1, $\lambda^*$ is calculated via the derivative of Tayler expansion. From analytical solution, we derive the optimal power allocation factor as $\lambda^*=0.662$, which is exactly same value of numerical optimum. In order to verify the guideline from Proposition~\ref{remark:1}, the simulation is conducted with the baseline ($\lambda=0.8$) in non-IID setting ($\alpha = 0.1$). Fig. ~\ref{fig:opt}(a) shows the result of experiment. The top-1 accuracy  with $\lambda^{*}$ shows $6.4\%$, $8.8\%$ higher accuracy than the top-1 accuracy in 0.5x and 1.0x, respectively. In other words, Proposition 1 provides influential guideline in SlimFL.

\BfPara{Optimal ST Weights}
In Proposition~\ref{prop:superposition}, SlimFL has tight bound when $w_1=\cdots=w_S=\frac{1}{S}$. Since we consider $S=2$ in our proposed scheme, all hyperparameters constituting ST should be $0.5$, i.e., $(w^*_1,w^*_2)=(0.5,0.5)$. To verify Proposition~\ref{prop:superposition}, we design baseline as $(w_1,w_2) = (0.3,0.7)$. Fig.~\ref{fig:opt}(b) shows performance difference according to different $\omega_i$. In optimal ST settings, top-1 accuracy achieves $78\%$ whereas, baseline achieves $69\%$. Thus, the guideline for ST positively effects the performance of SlimFL.

\BfPara{Scalability}
As Fig.~\ref{fig:Number of worker} represents, the accuracy of SLimFL improves as the number of federating devices increases. The SlimFL-0.5x accomplished the accuracy up to 79\%, and the SlimFL-1.0x accomplished the accuracy up to 85\%. In addition, with the non-iidness ($\alpha =0.1$), and with over the number of 70 federating local devices, SlimFL-0.5x shows higher accuracy than the SlimFL-1.0x with 20 federating local devices. Based on the experimental result, it is expected that optimality can be achieved by adjusting the number of local devices and the width through SlimFL adaptation when configuring an FL system based on non-IID datasets.

\begin{table}[t!]
    \caption{Accuracy under different channel conditions and~$\alpha$.}
    \label{tab:accuracy}
    \small\centering
    \resizebox{\linewidth}{!}{\begin{tabular}{c||ccc|ccc}
        \toprule[1pt]
        \multirow{3}{*}{\bf{Method}} &\multicolumn{6}{c}{\bf{Top-1 Accuracy (\%)}} \\ &\multicolumn{3}{c}{\bf{Good}}  & \multicolumn{3}{c}{\bf{Poor}}\\
         & $\alpha=0.1$ & $\alpha=1$ &  $\alpha=10$ & $\alpha=0.1$ & $\alpha=1$ &  $\alpha=10$ \\ \midrule
        SlimFL-0.5x & $54 \pm 2.2$ &  $ 83 \pm 1.0$ & $ 85 \pm 1.0$ & $56 \pm 2.4$ & $82 \pm 1.7$ & $ 85 \pm 1.1$ \\ 
        SlimFL-1.0x & $ 59  \pm  2.3$ &  $ 85 \pm 1.1$ & $ 87 \pm 1.1$ & $ 65 \pm 2.9$ & $ 84 \pm 1.4$ & $ 87 \pm 0.9$  \\
        Vanilla FL-0.5x & $ 45 \pm 5.9$ &  $ 84 \pm 1.1$ & $ 85 \pm 1.0$ & $ 39 \pm 8.3$ & $83 \pm 1.2$ & $ 85 \pm 0.9$\\
        Vanilla FL-1.0x & $69  \pm 5.8$ &  $ 85 \pm 4.0$ & $86  \pm  4.3$ & $ 55 \pm 9.2$ & $80 \pm  6.0$ & $ 82 \pm 4.7$ \\ \bottomrule[1pt]
    \end{tabular}}
\end{table}

\begin{figure}[t!]
\small\centering
\begin{tabular}{ccc}
    \multicolumn{3}{c}{\includegraphics[width=0.8\columnwidth]{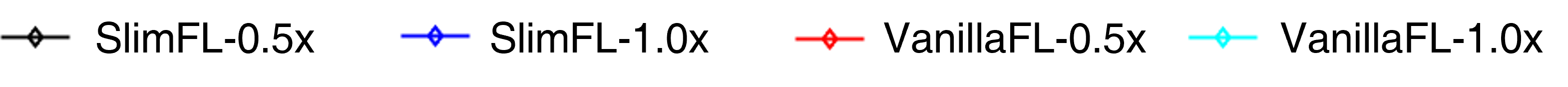}}\\
    \includegraphics[width=0.291\columnwidth]{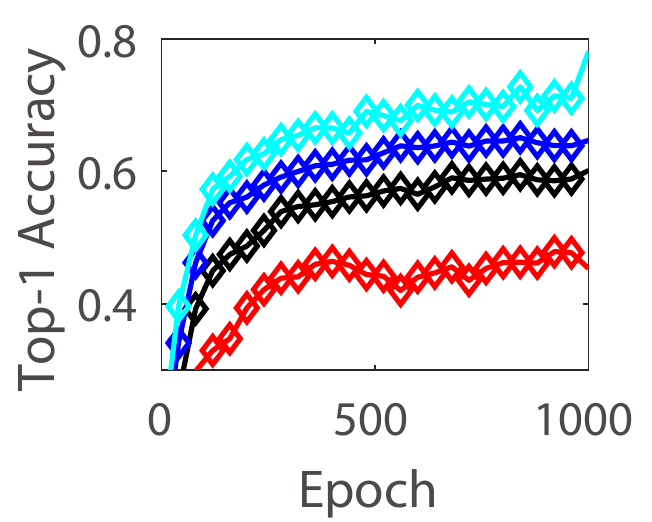}&
    \includegraphics[width=0.291\columnwidth]{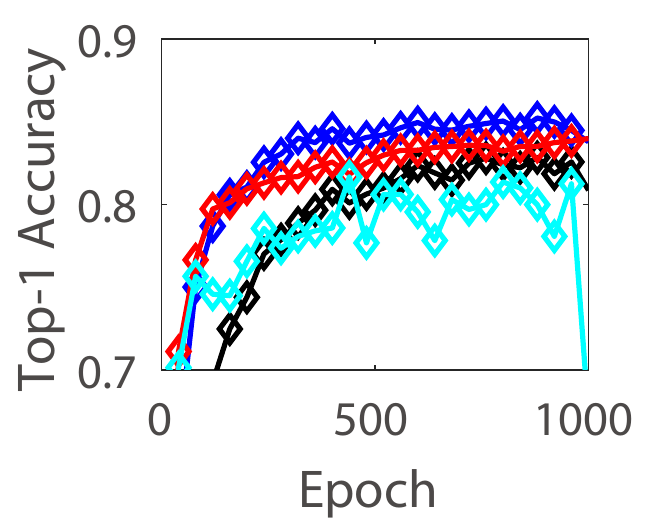}&
    \includegraphics[width=0.291\columnwidth]{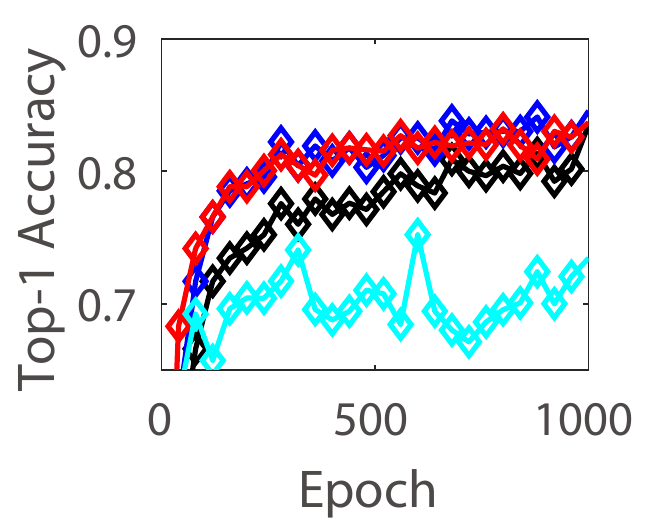}\\
    (a) $\sigma^2 = -40\mathrm{dB}$&
    (b) $\sigma^2 = -30\mathrm{dB}$&
    (c) $\sigma^2 = -20\mathrm{dB}$\\
    \includegraphics[width=0.291\columnwidth]{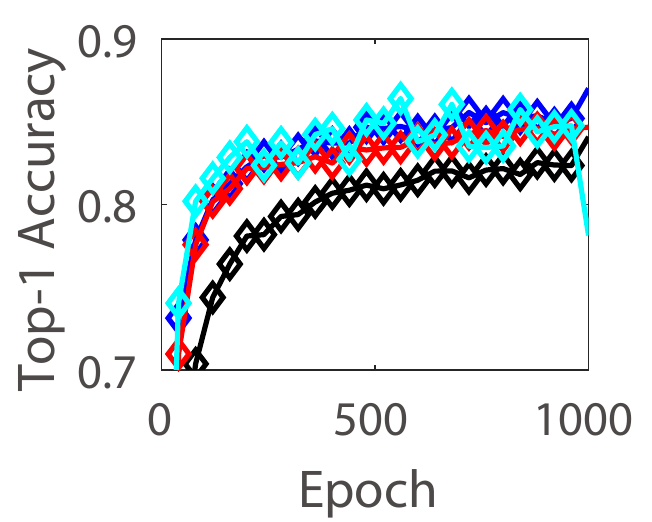}&
    \includegraphics[width=0.291\columnwidth]{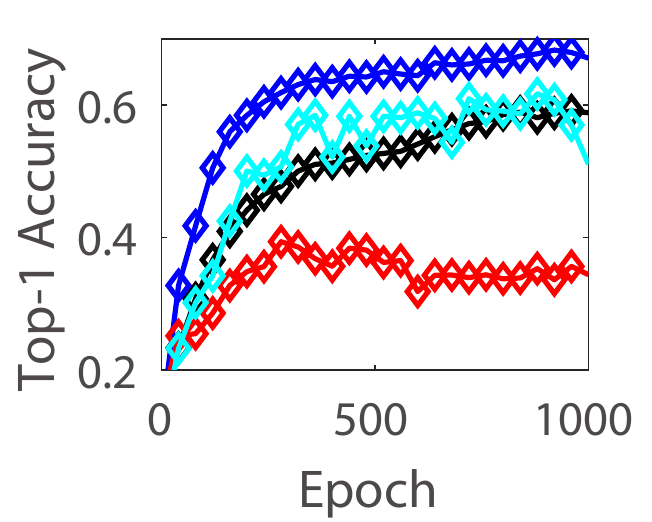}&
    \includegraphics[width=0.291\columnwidth]{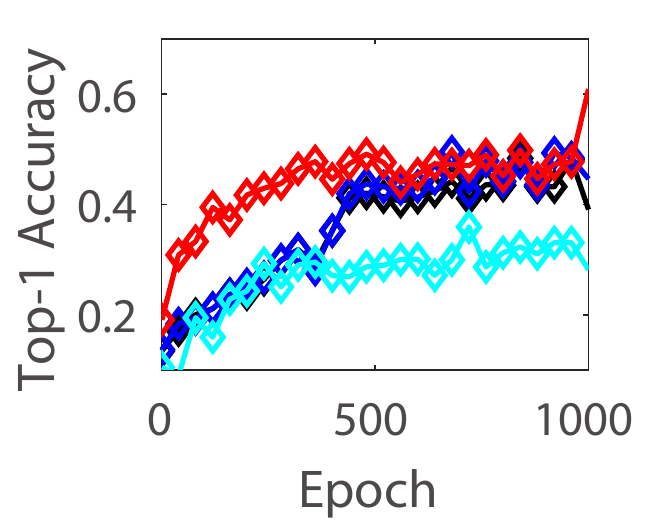}\\
    (d) $\sigma^2 = -40\mathrm{dB}$&
    (e) $\sigma^2 = -30\mathrm{dB}$&
    (f) $\sigma^2 = -20\mathrm{dB}$\\
\end{tabular}
\begin{tabular}{cc}
    \includegraphics[width=0.47\linewidth]{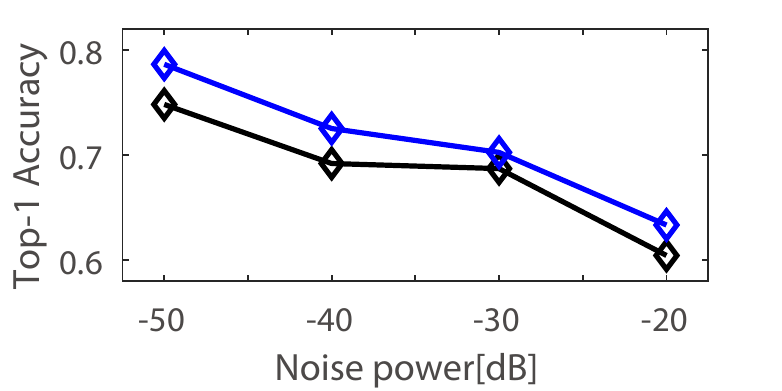}&
    \includegraphics[width=0.47\linewidth]{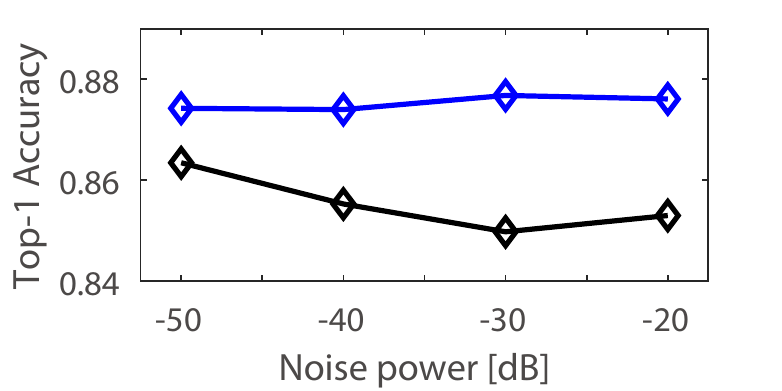}\\
    (g) $\alpha = 0.1$&
    (h) $\alpha = 10$ \\
\end{tabular}
    \caption{{Test accuracy} in various channel noise conditions (on average). (a--c) are with $\alpha=1$, (d--f) and (h) are with $\alpha=0.1$, (f) is with $\alpha=10$.}
    \label{fig:performance-good}
\end{figure}
\subsection{Performance of SlimFL}
We constructed an experiment to verify the performance of SlimfL compared to Vanilla FL in the environments with various communication conditions and non-iid settings.  

\BfPara{Robustness to Non-IID Data}
As illustrated in Fig.~\ref{fig:performance-good}(d--f) and Tab.~\ref{tab:accuracy}, SlimFL-0.5x shows a stable convergence under the conditions of $\alpha$. Vanilla FL-0.5x and Vanilla FL-1.0x exhibit the std of 8.3 and 9.2, in poor channel condition and with the non-IID dataset ($\alpha =0.1$). On the contrary, both SlimFL-0.5x and SlimFL-1.0x exhibit the std of 2.4 and 2.9 at top-1 accuracy. This tendency holds even when $\alpha = 1$, $\alpha = 10$. SlimFL-1.0x and SlimFL-0.5x exhibit lower variation than Vanilla FL-1.0x and Vanilla FL-0.5x. This underscores the robustness of SlimFL to non-IID data in poor channels.

\BfPara{Robustness to Poor Channels}
Fig.~\ref{fig:performance-good} and Tab.~\ref{tab:accuracy} show that both SlimFL and Vanilla FL achieve high accuracy in good channel conditions. However as the channel condition deteriorates from good to poor channels, Fig.~\ref{fig:performance-good}(c,f) and Tab.~\ref{tab:accuracy} illustrate that the maximum accuracy of Vanilla FL-1.0x at $\alpha=10$ drops from $86\%$ to $82\%$. Meanwhile, the accuracy of Slim-FL-1.0x keeps the same maximum accuracy $87$\% at $\alpha=10$ under both good and poor channels. What is more, at $\alpha=0.1$, SlimFL-1.0x even achieves $18$\% higher top-1 accuracy than Vanilla FL-1.0x that consumes more communication and computing costs. Furthermore, the std of Vanilla FL-1.0x's top-1 accuracy increases by up to $59$\% as channel condition deteriorates, whereas that of SlimFL increases by only up to $31$\%. These results advocate the robustness of SlimFL against poor channels, as well as its robustness to non-IID data distributions (low $\alpha$) and communication efficiency. 
\begin{table}[t!]
    \caption{Computing costs and transmission power of UL-MobileNet.}
    \label{tab:tab_cost}
    \small \centering
    \resizebox{0.75\columnwidth}{!}{\begin{tabular}{c|l|cc}
    \toprule[1pt]
      \multicolumn{2}{c|}{\bf{Description}}  & \bf{1.0x} & \bf{0.5x}  \\ \midrule
     \multirow{3}{*}{Computation}  &  MFLOPS / round & $2.76$ & $0.79$ \\ 
       &  \# of parameters & $4,586$ & $2,293$\\
       &  Bits / round & $172,688$ & $86,344$ \\\midrule
        \multicolumn{2}{c|}{Transmission Power ($P$) [mW]} & $132.1$  & $67.4$ \\ 
       \bottomrule[1pt]    
    \end{tabular}}
    \vspace{-3mm}
\end{table}
\begin{table}[t!]
    \caption{Transmission and Computing Costs per Communication Round.}
    \label{tab:energy}
    \scriptsize\centering
    \resizebox{\columnwidth}{!}{\begin{tabular}{c|c|c}
        \toprule[1pt]
        {\bf{Metric}} & \bf{SlimFL} & \bf{Vanilla FL-1.5x} \\\midrule
         Communication Cost [mW/Round]           & 199.5           & 399.1   \\
        {Computation Cost [MFLOPS/Epoch]} & {$3.56$} & $3.56$\\\bottomrule[1pt]
    \end{tabular}}
    \vspace{-3mm}
\end{table}
\begin{table}[t!]
    \caption{Successfully decoded bits of SlimFL, and Vanilla FL.}
    \label{tab:bits}
    \scriptsize\centering
    \resizebox{\columnwidth}{!}{\begin{tabular}{c|ccc|cc|cc}
        \toprule[1pt]
        \bf Decoding Success & \multicolumn{3}{c|}{\bf SlimFL} & \multicolumn{2}{c|}{\bf Vanilla FL-0.5x}  & \multicolumn{2}{c}{\bf Vanilla FL-1.0x}  \\
        \bf Bits [MBytes] &0.5x& 1.0x & drop & 0.5x & drop & 1.0x & drop \\ \midrule
        $\sigma^2=-30\mathrm{dB}$  & 1.96 & 198.45 & 5.46 & 102.21 & 0.72 & 200.30 & 5.56  \\
        $\sigma^2=-40\mathrm{dB}$   & 18.32 & 130.10 & 57.44 &93.87 & 9.06 & 144.93 & 60.96  \\ \bottomrule[1pt]
    \end{tabular}}
    \vspace{-3mm}
\end{table}

\subsection{Communication and Energy Efficiency}
In order to figure out the efficiency of communication and computation, we first calculate computation cost for feed-forwarding in UL-MobileNet~\cite{hernandez2020measuring} and communication per one communication round.

\BfPara{Communication Efficiency} 
The total amounts of transmitted bits between 10 devices and server in ideal channel conditions (i.e., always successful decoding) are 205.8MBytes for SlimFL and Vanilla FL-1.0x, and 102.9MBytes for Vanilla FL-0.5x.
Tab.~\ref{tab:bits} shows that SlimFL achieves up to $3.52$\% less dropped bits than Vanilla FL-1.0x, thanks to the use of SC and SD. The reduced dropped bits of SlimFL can be found by the successfully decoded bits of 0.5x models that cannot be simultaneously received under Vanilla FL-1.0x. Note that SlimFL decodes less 1.0x model bits than Vanilla FL-1.0x, as a part of transmission power of SlimFL is allocated to 0.5x models. In return, SlimFL not only receives 1.0x models but also 0.5x models simultaneously. The additionally received 0.5x models correspond to the LH parts of the 1.0x models, which therefore improve the accuracy and convergence speed of both 0.5x and 1.0x models. SlimFL enjoys the aforementioned benefits while consuming only the half of the transmission power and bandwidth compared to Vanilla FL-1.5x, as illustrated in Tab.~\ref{tab:bits}, corroborating its communication efficiency.

\BfPara{Energy Efficiency}
Thus far we have measured the performance of SlimFL after training with a fixed $1000$ epochs. Here, we measure the energy expenditure until convergence, where the training convergence is defined by the moment when the standard deviation (std) of test accuracy is below a target threshold and the minimum test accuracy becomes higher than the average test accuracy in $100$ consecutive rounds. To measure the convergence of models, we define the reference values of the mean $\mu_{\mathsf{Ref}}$ as 80\% and $\sigma_{\mathsf{Ref}}$ as 7.2\%, respectively. We define the convergence when average of Top-1 accuracy for 100 consecutive epochs is higher than $\mu_{\mathsf{Ref}}$, and the average std is lower than $\sigma_{\mathsf{Ref}}$.  
 Given the communication and computing energy costs per round in Tab.~\ref{tab:energy}, Tab.~\ref{tab:energy2} compares the total energy costs of SlimFL and Vanilla FL-1.5x until convergence. The results show that on average, SlimFL achieves $3.6$x less total computing cost with $2.9$x lower total communication cost until convergence. Such higher energy efficiency comes from the faster convergence of SlimFL even under non-IID and/or poor channel conditions due to SC and SD.

\begin{table}[t!]
    \caption{Total Computation cost and Transmission Power of SlimFL and Vanilla FL-1.5x in Various non-IIDness ($\alpha=0.1,1.0,10)$.}
    \label{tab:energy2}
    \scriptsize\centering
    \resizebox{\columnwidth}{!}{\begin{tabular}{c|c||cc|cc}\toprule[1pt]
        \multirow{2}{*}{\bf{Metric}} &  \multirow{2}{*}{\bf{non-IIDness}} & \multicolumn{2}{c|}{\bf{SlimFL}} & \multicolumn{2}{c}{{\bf Vanilla FL-1.5x}}\\
         &  &  \bf{Good} & \bf{Poor} &  \bf{Good} & \bf{Poor}\\\midrule
         
        \multirow{3}{*}{\shortstack{Communication \\ Cost [W]}} & \bf{$\alpha=0.1$} & 71.0 & 57.3 & 158.8 & 196.8 \\
         & \bf{$\alpha=1.0$} & 8.5 & 10.4 & 15.8 & 36.7 \\
         & \bf{$\alpha=10$}  & 3.03 & 3.51 & 10.2 & 25.4 \\\midrule
         
        \multirow{3}{*}{\shortstack{Computation \\ Cost [GFLOPS]}} & \bf{$\alpha=0.1$} & 1.27 & 1.02 & 1.88 & 2.41 \\
         & \bf{$\alpha=1.0$} & 0.15 & 0.18 & 0.22 & 0.51\\
         & \bf{$\alpha=10$} & 0.05 & 0.06 & 0.14 & 0.35\\ \bottomrule[1pt]
     \end{tabular}}
\end{table}

\section{Conclusion}\label{sec:6}
Existing FL solutions cannot cope flexibly with different devices having heterogeneous levels of available energy and channel throughput.
To tackle this problem, we propose a novel framework of FL over SNNs, SlimFL, by developing ST for local SNN training and exploiting SC for the trained model aggregation. Extensive experiments verify that SlimFL is a communication and energy efficient solution under various communication environments and data distributions. Particularly under poor channel conditions and non-IID data distributions, SlimFL even achieves higher accuracy and faster convergence with lower energy expenditure than its vanilla FL counterpart consuming $2$x more communication resources. Additionally incorporating more width configurations and local iterations could be interesting topics for future research.

\section*{Acknowledgment}
This research was funded by IITP 2021-0-00467. The first three authors are equally contributed (first authors).  
S. Jung, J. Park, and J. Kim are corresponding authors.

\appendices
\section{Local SNN Training Algorithm}\label{sec:Appendix-1}

\begin{algorithm}[ht]\label{alg:slim}
\caption{\small SlimTrain~\cite{yu2018slimmable}}
\scriptsize 
    {Define \textit{switchable width list} for slimmable network \(M\), for example, {\([0.25, 0.5, 0.75, 1.0]\times\)}.}\\
    {Initialize shared convolutions and fully-connected layers for slimmable network \(M\).}\\
    {Initialize independent batch normalization parameters for each \textit{width} in \textit{switchable width list}.}\\
    \For {$i = 1, ..., n_{iters}$}
        {Get next mini-batch of data \(x\) and label \(y\).\\
        Clear gradients of weights, \(optimizer.zero\_grad()\).\\
        \For {\textit{width} in \textit{switchable width list}}
            {
            Switch the batch normalization parameters of current width on network \(M\).\\
            Execute sub-network at current width, \(\hat{y} = M'(x)\).\\
            Compute loss, \(loss = criterion(\hat{y}, y)\).\\
            Compute gradients, \(loss.backward()\).
            }
        Update weights, \(optimizer.step()\).
        }
\end{algorithm}

\begin{algorithm}[ht]\label{alg:usslim}
\scriptsize 
    Define \textit{width range}, for example, { $[0.25,0.5,0.75,1.0]$x}.\\
    Define \textit{n} as number of sampled widths per training iteration, for example, $n=4$.\\
    Initialize training settings of shared network \(M\).\\
    \For {$(t = 1, ..., T_{iters})$}
    {
        Get next mini-batch of data $x$ and label $y$).\\
        Clear gradients, $optimizer.zero\_grad()$.\\
        Execute full-network, $y' = M(x)$.\\
        Compute loss, $loss = criterion(y', y)$.\\
        Accumulate gradients, $loss.backward()$.\\
        Stop gradients of $y'$ as label, $y' = y'.detach()$.\\ 
        Add smallest width to \textit{width samples}.\\
        \For {\textit{width} in \textit{width samples}}{
            Execute sub-network at \textit{width}, $\hat{y} = M'(x)$.\\
            Compute loss, $loss = criterion(\hat{y}, y')$.\\
            Accumulate gradients, $loss.backward()$.
            }
        Update weights, \(optimizer.step()\).
    }
\caption{\small USTrain~\cite{ICCV2019_USlimmable}}
\end{algorithm}

\begin{table}[t!]
    \caption{List of Notations}
    \label{tab:notation-convergence}
    \footnotesize\centering
    \resizebox{\columnwidth}{!}{\begin{tabular}{c|l}
        \toprule[1pt]
        \textbf{Notation} & \textbf{Description}\\\midrule[1pt]
        ${T}$ & Total iteration steps.\\
        ${S}$ & The number of SNN width configurations.\\
        $\theta^{G}$ & Global model parameter vector. \\
        $\theta^{k}$ & Local model parameter vector of the $k$-th device. \\
        $\mathcal{K}$ & A set of devices $(\mathcal{K}=\{1,\cdots, k, \cdots, K\})$.\\\midrule
        $\Xi$ & A binary mask to extract weight parameters of an LH segment. \\
        $~~\Xi^{-1}$ & A binary mask to extract weight parameters of an RH segment. \\
        $\mathsf{H}$ & A set of successfully decoded LH segments.\\
        $\mathsf{F}$ & A set of successfully decoded RH segments.\\
        $n_\mathsf{L}$ & The number of successfully decoded LH segments.\\
        $n_\mathsf{R}$ & The number of successfully decoded RH segments.\\
        $p_{1}$ & The decoding success probability of an LH segment.\\
        $p_{2}$ & The decoding success probability of an RH segment.\\\midrule     
        $\bm{Z}$ & Entire dataset.\\
        $\Xi_{{i}}$ & A binary mask to extract weight parameters of the $i$-th smallest model. \\
        $\omega_{i}$ & Positive constant for updating full model via the $i$-th smallest model. \\ 
        $\eta_t$ & Learning rate at the iteration $t$.\\
        $\zeta^k_t$ & The local data sampled from $k$-th user at the iteration $t$.\\  
        \bottomrule[1pt]
     \end{tabular}}
\end{table}

SlimTrain in~\cite{yu2018slimmable} and USTrain in~\cite{ICCV2019_USlimmable} are described by \textbf{Algorithm~\ref{alg:slim}} and \textbf{Algorithm~\ref{alg:usslim}}, respectively.
In essence, SlimTrain and USTrain both utilize alternating methods, and USTrain utilizes IPKD and the sandwich rule that are effective when each SNN has more than two width configurations. We only consider $2$ width configurations, making the sandwich rule unfit for our case. Therefore, ignoring the sandwich rule, our proposed SUSTrain only utilizes IPKD while additionally exploiting ST. Fig.~\ref{fig:trainmethod} shows that SUSTrain outperforms SlimTrain and USTrain under both IID and non-IID data.

\bibliographystyle{IEEEtran}

\end{document}